\pdfoutput=1
\documentclass[10pt]{article}

\usepackage[T1]{fontenc}

\usepackage{stmaryrd}
\usepackage{amsthm}
\usepackage{amsmath}
\usepackage{amssymb}
\usepackage{mathrsfs}
\usepackage{soul}
\usepackage{pgf,tikz}
\usepackage{graphicx}
\usepackage{xcolor}
\usetikzlibrary{arrows}
\usepackage{listings}
\usepackage{hyperref}
\usepackage{cases}
\usepackage{mathabx}
\usepackage{algorithm}
\usepackage{algorithmic}
\usepackage[nottoc,numbib]{tocbibind}
\usepackage{natbib}
\usepackage{mdframed}
\usepackage{float}
\usepackage{enumitem}
\usepackage{multicol}
\usepackage{wrapfig}
\usepackage{booktabs}
\usepackage{enumitem}

\theoremstyle{plain}
\newtheorem{thm}{Theorem}[section]

\newtheorem{prop}[thm]{Proposition}

\theoremstyle{definition}
\newtheorem{defn}{Definition}[section]

\theoremstyle{remark}
\newtheorem{rem}{Remark}

\DeclareMathOperator*{\argmax}{arg\,max}
\DeclareMathOperator*{\argmin}{arg\,min}
\newcommand{\R}{\mathbb{R}}

\newcommand{\E}{\mathbf{E}}

\newcommand{\vertiii}[1]{{\left\vert\kern-0.25ex\left\vert\kern-0.25ex\left\vert #1 
\right\vert\kern-0.25ex\right\vert\kern-0.25ex\right\vert}}

\newcommand{\cC}{\mathcal{C}}
\newcommand{\cE}{\mathcal{E}}
\newcommand{\cN}{\mathcal{N}}
\newcommand{\cP}{\mathcal{P}}
\newcommand{\cX}{\mathcal{X}}
\newcommand{\cY}{\mathcal{Y}}

\newcommand{\ud}{\mathrm{d}}

\newmdtheoremenv{propbox}{Proposition}
\newmdtheoremenv{lembox}{Lemma}
\newmdtheoremenv{corbox}{Corollary}
\newmdtheoremenv{defnbox}{Definition}

\hypersetup{colorlinks=true,
            urlcolor=blue,
            citecolor=blue
            } 

\usepackage{microtype}
\usepackage{graphicx}
\usepackage{subfigure}
\usepackage{booktabs}

\usepackage[utf8]{inputenc} 
\usepackage{url}            
\usepackage{nicefrac}       

\title{Learning with Differentiable Perturbed Optimizers}

\author{%
  Quentin Berthet\thanks{
  Google Research, Brain team, Paris} \thanks{
  \texttt{\{qberthet, mblondel, oliviert, cuturi, jpvert\}@google.com}
  }\and 
  Mathieu Blondel\footnotemark[1] \footnotemark[2]
  \and
  Olivier Teboul\footnotemark[1] \footnotemark[2]
  \and
  Marco Cuturi\footnotemark[1] \footnotemark[2]
  \and
  Jean-Philippe Vert\footnotemark[1] \footnotemark[2]
  \and
  Francis Bach\thanks{
  INRIA - DI, ENS,
PSL Research University Paris,
  \texttt{francis.bach@inria.fr} 
  }
}

\begin{document}
\maketitle

\begin{abstract}
Machine learning pipelines often rely on optimization procedures to make discrete decisions (e.g., sorting, picking closest neighbors, or shortest paths). Although these discrete decisions are easily computed, they break the back-propagation of computational graphs. In order to expand the scope of learning problems that can be solved in an end-to-end fashion, we propose a systematic method to transform optimizers into operations that are differentiable and never locally constant. Our approach relies on stochastically perturbed optimizers, and can be used readily together with existing solvers. Their derivatives can be evaluated efficiently, and smoothness tuned via the chosen noise amplitude. We also show how this framework can be connected to a family of losses developed in structured prediction, and give theoretical guarantees for their use in learning tasks. We demonstrate experimentally the performance of our approach on various tasks.
\end{abstract}

\section{Introduction}
Many applications of machine learning benefit from the possibility to train by gradient descent compositional models using end-to-end differentiability.  Yet, there remain fields where discrete decisions are required at intermediate steps of a data processing pipeline (e.g., in robotics, graphics or biology). This is the result of many factors: discrete decisions provide a much sought-for interpretability, and discrete solvers are built upon decades of advances in combinatorial algorithms~\citep{schrijver2003combinatorial} for quick decisions (e.g., sorting, picking closest neighbors, exploring options with beam-search, or with shortest paths problems). These discrete decisions can easily be computed in a forward pass. Their derivatives with respect to inputs are however degenerate: small changes in the inputs either yield no change or discontinuous changes in the outputs. Discrete solvers thus break the back-propagation of computational graphs, and cannot be incorporated in end-to-end learning.

In order to expand the set of operations that can be incorporated in differentiable models, we propose and investigate a new, systematic method to transform discrete optimizers into differentiable operations. 
Our approach builds upon the method of stochastic perturbations, the theory of which was developed and applied to several tasks of machine learning recently; see \cite{hazan2016perturbations}. In a nutshell, we perturb the inputs of a discrete solver with  random noise, and consider the perturbed solutions of the problem. The method is both easy to analyze theoretically and simple to implement. We show that the formal expectations of these perturbed solutions are  never locally constant and everywhere differentiable, with successive derivatives being expectations of simple expressions.

{\bf Related work.} Our work is part of growing efforts to modify operations to make them differentiable. Several works have studied the introduction of regularization in the optimization problem to make the argmax differentiable. These works are usually problem-specific, since a new optimization problem needs to be solved. Examples include assignments~\citep{adams2011ranking}, optimal transport \citep{bonneel2016wasserstein, cuturi2013sinkhorn}, differentiable dynamic programming \citep{mensch_2018},
differentiable submodular optimization \citep{djolonga_2017}. A generic approach is SparseMAP \citep{niculae_2018}, based on Frank-Wolfe or active-set algorithms for solving, and on implicit differentiation for Jacobian computation.
Like our proposal, SparseMAP only requires access to a linear maximization oracle. However, it is sequential in nature, while our approach is trivial
to parallelize. In \citep{agrawal2019differentiable}, implicit differentiation on solutions of convex optimization is analyzed. They express the derivatives of the argmax exactly, leading to zero Jacobian almost everywhere when optimizing over polytopes. \citet{vlastelica2019differentiation} proposed to interpolate in a piecewise-linear manner between locally constant regions. The aim is to keep the same value for the Jacobian of the argmax for a large region of inputs, allowing for zero Jacobians as well.

An example of expectation of a perturbed argmax, commonly known as the ``Gumbel trick'',
dates back to \citet{gumbel_1954}, and random choice models \citep{luce1959individual,mcfadden1973conditional,guadagni1983logit}. It is exploited in online learning and bandits to promote exploration, and induce robustness to adversaries (see, e.g., \citep{abernethy2016perturbation} for a survey). It is used for action spaces that are combinatorial in nature \citep{neu2016importance}, and used together with a softmax to obtain differentiable sampling \citep{jang_2016,maddison_2016}, and with distributions from extreme value theory \citep{balog_2017}. 

The use of perturbation techniques as an alternative to MCMC techniques for sampling was pioneered by \citet{papandreou_2011}. They are used to compute expected statistics arising in gradients of conditional random fields. They show exactness for the fully perturbed (but intractable case) and propose ``low-rank'' perturbations as an approximation. These results are extended in
\citep{hazan_2012}, proving that the expected maximum with low-rank perturbations provides an upper-bound on the log partition, and replacing the log partition in conditional random fields loss by that expectation. Their results, however, are limited to discrete product spaces.  New lower bounds on the partition function are derived in \citep{hazan_2013}, as well as a new unbiased sequential sampler for the Gibbs distribution based on low-rank perturbations. These results were further refined in \citep{gane_2014} and \citep{orabona_2014}, and these bounds further studied in 
\citep{shpakova_2016}, who proposed a doubly stochastic scheme. Apart from \citep{lorberbom_2019}, who use a finite difference method, we are not aware of any prior work using perturbation techniques
to differentiate through an \textbf{argmax}. As reviewed above, all papers focus on (approximately) sampling from the Gibbs distribution, upper-bounding the log partition function, or differentiating through the \textbf{max}.
\vspace{0.2cm}

{\bf Contributions.} We make the following contributions:
\vspace{0.2cm}

- We propose a new general method transforming discrete optimizers, inspired by the stochastic perturbation literature. This versatile method applies to any blackbox solver without ad-hoc modifications.
\vspace{0.2cm}

- Our stochastic smoothing allows \textbf{argmax} differentiation, through the formal perturbed maximizer. Its Jacobian is well-defined and non-zero everywhere, thereby avoiding vanishing gradients. 
\vspace{0.2cm}

- The successive derivatives of the perturbed maximum and argmax are expressed as simple expectations, which are easy to approximate with Monte-Carlo methods.
\vspace{0.2cm}

- Our method yields natural connections to the recently-proposed Fenchel-Young losses by \citet{fy_losses_journal}. We show that the equivalence via duality with regularized optimization makes these losses natural. 
\vspace{0.2cm}

- We propose a doubly stochastic scheme for their minimization in learning tasks, and we demonstrate our method on structured prediction tasks, in particular ranking (permutation prediction), for which conditional random fields and the Gibbs distribution are intractable.

\section{Perturbed maximizers}

\begin{wrapfigure}[14]{r}[-12pt]{0.45\textwidth}
\vspace*{-.75cm}
\begin{center}
\includegraphics[width = 0.41\textwidth]{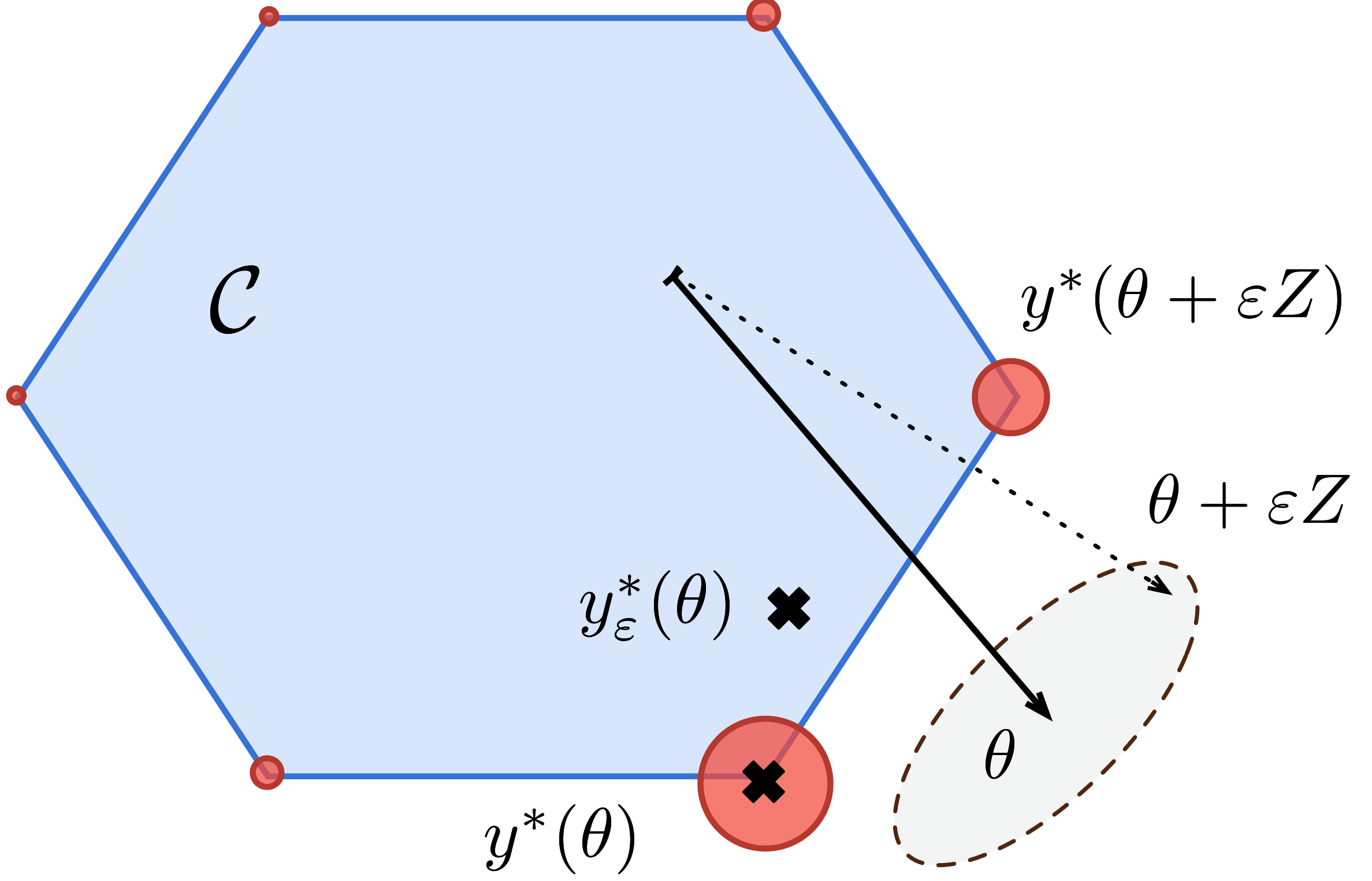}
\end{center}
\caption{\label{fig:perturb} Stochastic smoothing yields a perturbed optimizer $y^*_\varepsilon$ in expectation.
}
\end{wrapfigure}

Given a finite set of distinct points $\cY \subset \R^d$ and~$\cC$ its convex hull, we consider a general discrete optimization problem parameterized by an input $\theta\in\R^d$ as follows:
\begin{equation}
\label{EQ:optDISC}
\!\!F(\theta) = \max_{y \in \cC}\  \langle y , \theta \rangle\, , \  y^*(\theta) = \argmax_{y \in \cC}\  \langle y , \theta \rangle\ .
\end{equation}
As we discuss below, this formulation encompasses a variety of discrete operations commonly used in machine learning. In all cases, $\cC$ is a convex polytope and these problems are linear programs (LP). For almost every $\theta$, the argmax is unique, and $y^*(\theta) = \nabla_\theta F(\theta)$. While widespread, these functions do not have the convenient properties of blocks in end-to-end learning architectures, such as smoothness or differentiability. In particular, $\theta \mapsto y^*(\theta)$ is piecewise constant: its gradient is zero almost everywhere, and undefined otherwise. To address these issues, we simply add to $\theta$ a random noise vector $\varepsilon Z$, where $\varepsilon>0$ is a temperature parameter and $Z$ has a positive and differentiable density $\ud\mu(z) \propto \exp(-\nu(z)) \ud z$ on $\R^d$, so that $y^*(\theta + \varepsilon Z)$ is almost surely (a.s.) uniquely defined. This induces a probability distribution $p_\theta$ for $Y \in \cY$ given by $p_\theta(y) = P(y^*(\theta + \varepsilon Z)=y)$; see Figure~\ref{fig:perturb}. 

This creates a general and natural model on the variable $Y$, when observations are solutions of optimization problems, with uncertain costs. It enables the modeling of phenomena where agents chose an optimal $y \in \cC$ based on uncertain knowledge of $\theta$. We view this as a generalization, or alternative to the Gibbs distribution, rather than an approximation thereof.

Taking expectations with respect to the random perturbation leads to smoothed versions of $F$ and $y^*$:
\begin{defn}\label{def:perturb}
For all $\theta \in \R^d$, and $\varepsilon>0$, we define the {\bf perturbed maximum} as
\[ \textstyle
F_\varepsilon(\theta) = \E[F(\theta+\varepsilon Z)] = \E[\max_{y \in \cC} \ \langle y , \theta +\varepsilon Z\rangle],
\]
and, the {\bf perturbed maximizer} as 
\[
y^*_\varepsilon(\theta) = \E_{p_\theta(y)}[Y] = \E[\argmax_{y \in \cC}\langle y , \theta +\varepsilon  Z\rangle] = \E[\nabla_\theta \max_{y \in \cC}\langle y , \theta +\varepsilon  Z\rangle] = \nabla_\theta F_\varepsilon (\theta)\, .
\]
\end{defn}
Models of random optimizers for linear problems with perturbed inputs are the subject of a wide litterature in machine learning, under the name of ``perturb-and-MAP'' \cite{papandreou_2011,hazan_2012}, and perturbed leader method in online learning \citep{hannan57,kalai05,AbeLeeSin14}. We refer to it here as the {\em perturbed model}. 

\paragraph{Broad applicability.}

Many operations used in machine learning can be written in the form of Eq.~\eqref{EQ:optDISC} and are thus part of our framework. 
Indeed, for any score function $s:\cY \to \R$, the problem $\max_{y \in \cY} s(y)$, can at least be written as a linear program (LP) in Eq.~\eqref{EQ:optDISC}, for some embedding of the set $\cY$. We emphasize that the LP structure need not be known to use the perturbed maximizers. In our experiments, we focus on the following three tasks (see Appendix~\ref{APP:examples} for more examples).
\vspace{0.2cm}

\noindent
{\em Maximum.} The max function from $\R^d$ to $\R$, that returns the largest among the $d$ entries of a vector $\theta$ is commonly used for $d$-way multiclass classification. It is equal to $F(\theta)$ over the unit simplex $\cC = \{y \in \R^d\, : \, y \ge 0\, , \;\mathbf{1}^\top y = 1\}$. 
The computational cost is $O(d)$.
On this set, using Gumbel noise yields the Gibbs distribution for $p_\theta$ (see below).
\vspace{0.2cm}

\noindent
{\em Ranking.} The function returning the ranks (in descending order) of a vector $\theta \in \R^d$ can be written as the argmax of a linear program over the {\em permutahedron}, the convex hull of permutations of any vector $v$ with distinct entries $\cC = \cP_v = \text{cvx}\{P_\sigma v \, : \, \sigma \in \Sigma_d\}$. The computational cost is $O(d \log d)$, using a sort.
\vspace{0.2cm}

\noindent
{\em Shortest paths.}
For a graph $G=(V,E)$ and positive costs over edges $c \in \R^E$, the problem of finding a shortest path (i.e., with minimal total cost) from vertices $s$ to $t$ can be written in our setting with $\theta = -c$ and $\cC = \{y \in \R^E : y \ge 0\,, (\mathbf{1}_{\to i} - \mathbf{1}_{i \to})^\top y = \delta_{i=s} - \delta_{i=t} \}$. The computational cost is $O(|E|)$, using dynamic programming.


\paragraph{A generalization of Gumbel-max.}

An example of this setting is well-known: when $\cY$ is the set of one-hot-encoding of $d$ classes, $\cC$ is the unit simplex, and $Z$ has the Gumbel distribution \citep{gumbel_1954}. In that case it is well-known that $p_\theta$ is the Gibbs distribution, proportional to $\exp(\langle y , \theta \rangle/\varepsilon)$, $F_\varepsilon(\theta)$ is the log-sum-exp function of $\theta$, and $y^*_\varepsilon(\theta)$ is the vector of softmax (or exponential weights) of the components of~$\theta$. Our model is therefore a \emph{generalization} of the Gumbel-max setting.
As $F_\varepsilon$ generalizes the log-sum-exp function for Gumbel noise on the simplex, its dual $\Omega$ is a \emph{generalization} of the negative Shannon entropy (which is the Fenchel dual of the log-sum-exp function).
We show this connection, and that the perturbed maximizer can also be defined as the solution of a convex problem, by Fenchel-Rockafellar duality in Proposition~\ref{PRO:duality} below. The following table summarizes those parallels. Our framework generalizes these ideas, and proposes to exploit the ease of simulation of $p_\theta$ (rather than the explicit forms of Gibbs distributions) for applications in machine learning tasks.
\begin{center}
\begin{tabular}{ |p{3.8cm} |p{3.8cm}|p{4.2cm}| }
\hline
  & Gumbel-max & General perturbed optimizer \\[.1cm] 
  \hline 
  &&\\[-.2cm] 
 noise distribution & $Z_i$ independent Gumbel & $Z\sim \mu$, general random \\[.2cm] 
 domain $\cC$& unit simplex $\Delta^n$  & general polytope: $\text{cvx}(\cY)$ \\[.2cm]
 argmax distribution &$p_{{\sf Gibbs}, \theta} \propto \exp(\langle y , \theta \rangle/\varepsilon)$  & $p_\theta$, no closed form\\[.2cm] 
expectation of maximum & log-sum-exp of $\theta$  & general $F_\varepsilon(\theta)$ \\[.2cm] 
convex regularizer & Shannon negentropy  & general $\Omega = (F_\varepsilon)^\star$\\
\hline
\end{tabular}
\end{center}

\vspace{.1cm}
\begin{prop}
\label{PRO:duality}
Let $\Omega$ be the Fenchel dual of $F_1$, with domain $\cC$. We have that
\begin{equation}
\label{EQ:regul}
y^*_\varepsilon(\theta) = \argmax_{y \in \cC}\big\{\langle y , \theta \rangle - \varepsilon \, \Omega(y) \big\}\, .
\end{equation}
\end{prop}

\paragraph{Differentiation and associated loss function.}


While these connections have been studied before \citep{hazan2016perturbations,AbeLeeSin14,abernethy2016perturbation}, we provide two key new insights.
First, the perturbed model allows to take derivatives with respect to the input~$\theta$ of $F_\varepsilon$ and of $y^*_\varepsilon$ (Proposition~\ref{PRO:fundPROP}). These derivatives are also easily expressed as expectations involving $F$ and $y^*$ with noisy inputs, as  discussed in Section~\ref{sec:diff}. In turn, this yields fast computational methods for these functions and their derivatives. 
Second, by the duality point of view describing $y^*_\varepsilon$ as a regularized maximizer, there exists a natural convex loss for this model that can be efficiently optimized in $\theta$, for data $y_i \in \cY$. We describe this formalism in Section~\ref{SEC:FY}, and apply it in experiments in Section~\ref{SEC:expe}.

\paragraph{Properties of the model.}

This model modifies the maximum and maximizer by perturbation. Because of the simple action of the stochastic noise , we can analyze their properties precisely.

\begin{prop}
\label{PRO:fundPROP}
Assume $\cC$ is a convex polytope with non-empty interior, and $\mu$ has positive differentiable density. The perturbed model $p_\theta$ and the associated functions $F_\varepsilon$, $\Omega = (F_\varepsilon)^\star$, and $y^*_\varepsilon$ have the following properties, for $R_\cC=\max_{y \in \cC}\|y\|$ and $M_\mu = \E[\|\nabla_z\nu(Z)\|^2]^{1/2}$:
\begin{itemize}[topsep=0pt,itemsep=2pt,parsep=2pt,leftmargin=5pt]
    \item[-] $F_\varepsilon$ is strictly convex, twice differentiable, $R_\cC$-Lipschitz-continuous and its gradient is $R_\cC M_\mu/\varepsilon$-Lipschitz-continuous. Its dual $\Omega$ is $1/(R_\cC M_\mu)$-strongly convex, differentiable, and Legendre-type. 
    \item[-] For all $\theta \in \R^d$, $y^*_\varepsilon(\theta)$ is in the interior of $\cC$ and $y^*_\varepsilon$ is differentiable in $\theta$. 
    \item[-] Impact of $\varepsilon>0$: we have $F_\varepsilon(\theta) =\varepsilon F_1\big(\frac{\theta}{\varepsilon}\big), \, F_\varepsilon^*(y) = \varepsilon \Omega(y) , \, y^*_{\varepsilon} (\theta) = y^*_{1}\big(\frac{\theta}{\varepsilon}\big)$.
\end{itemize}
\end{prop}

We develop in further details the simple expressions for derivatives of $F_\varepsilon$ and $y^*_\varepsilon$ in Section~\ref{sec:diff}. By this proposition, since $F_\varepsilon$ is strictly convex, it is nowhere locally linear, so $y^*_\varepsilon$ is nowhere locally constant. Formally, $y^*_\varepsilon = \nabla_\theta F_\varepsilon$ is a {\em mirror map}, a one-to-one mapping from $\R^d$ unto the interior of $\cC$. The gradient of $\varepsilon \Omega$ is its functional inverse, by convex duality between these functions (see, e.g., surveys \citep{wainwright2008graphical,bubeck2015convex} and and references therein).
\begin{rem}
\label{REM:empty}
For these properties to hold, it is crucial that $\cC$ has non-empty interior, i.e., that $\cY$ does not lie in an affine subspace of lower dimension. To adapt to cases where $\cC$ lies in a subspace, we consider the set of inputs $\theta$ up to vectors orthogonal to $\cC$, or represent $\cY$ in a lower-dimensional subspace. As an example, over the unit simplex and Gumbel noise, the log-sum-exp is not strictly convex, and in fact linear along the all-ones vector $\mathbf{1}$. 
In such cases, the model is only well-specified in $\theta$ up to the space orthogonal to $\cC$, which does not affect prediction tasks.
\end{rem}

For any positive temperature $\varepsilon$, these properties imply that there is an informative, well-defined, and nonzero gradient in $\theta$. They also imply the limiting behavior at extreme temperatures.
\begin{prop}
\label{PRO:temperature}
With the conditions of Proposition~\ref{PRO:fundPROP}, for~$\theta$ such that $y^*(\theta)$ is a unique maximum:

For $\varepsilon \to 0$, $F_\varepsilon(\theta) \to F(\theta)$ and $y^*_\varepsilon(\theta) \to y^*(\theta)$. For $\varepsilon \to \infty$, $y^*_\varepsilon(\theta) \to y^*_1(0) = \argmin_{y \in \cC} \Omega(y)$.

For every $\varepsilon>0$, we have $F(\theta) - F_\varepsilon(\theta) \le C \varepsilon$ and $\langle y^*(\theta) , \theta \rangle - \langle y^*_\varepsilon(\theta) , \theta \rangle \le C' \varepsilon$, for $C, C'> 0$.
\end{prop}
The properties of the distributions $p_\theta$ in this model are well studied in the perturbations literature (see, e.g., \citep{hazan2016perturbations} for a survey). They notably do not have a simple closed-form expression, but can be very easy to sample from. By the argmax definition, simulating $Y \sim p_\theta$, only requires to sample $\mu$ (e.g., Gaussian, or vector of i.i.d.~Gumbel), and to solve the original optimization problem. It is the case in the applications we consider (e.g., max, ranking, shortest paths). This is in stark contrast to the Gibbs distribution, which has the opposite properties.

\section{Differentiation of soft maximizers}\label{sec:diff}

As noted above, for the right noise distributions, the perturbed maximizer $y^*_\varepsilon$ is differentiable in its inputs, with non-zero Jacobian. It is based on integration by parts, not on finite differences as in \citep{lorberbom_2019}. 
\begin{prop}{\citep[][Lemma 1.5]{abernethy2016perturbation}}
\label{PRO:ipp}
For noise $Z$ with distribution $\ud\mu(z) \propto \exp(-\nu(z)) \ud z$ and twice differentiable $\nu$, the following holds (with $J_\theta \, y_\varepsilon^*(\theta)$ the Jacobian matrix of $y_\varepsilon^*$ at $\theta$):
\begin{align*}
F_\varepsilon(\theta) &= \E[F(\theta+ \varepsilon Z)] \,, \quad y_\varepsilon^*(\theta) = \nabla_\theta F_\varepsilon(\theta)=  \E[y^*(\theta+\varepsilon Z)] = \E[F(\theta+\varepsilon Z) \nabla_z \nu(Z)/\varepsilon] \,, \\
J_\theta \, y_\varepsilon^*(\theta)&= \E[y^*(\theta+\varepsilon Z) \nabla_z \nu(Z)^\top/\varepsilon] = \E[F(\theta+\varepsilon Z) (\nabla_z \nu(Z) \nabla_z \nu(Z)^\top - \nabla_z^2 \nu(Z))/\varepsilon^2 ]\, .
\end{align*}
\end{prop}

The derivatives are simple expectations. We discuss in the following subsection efficient techniques to evaluate in practice $y^*_\varepsilon(\theta)$ and its Jacobian, or to generate stochastic gradients, based on these expressions.
\begin{rem}
\label{REM:learning}
Being able to compute the perturbed maximizer and its Jacobian allows to optimize functions that depend on~$\theta$ through $y^*_\varepsilon (\theta)$. This can be used to alter the costs to promote solutions with certain desired properties. Moreover, in a supervised learning setting, this allows to train models containing blocks with inputs $\theta = g_w(x)$, for some feature vector $x$, 
by minimizing a loss $\ell$ between the perturbed maximizer $y^*_\varepsilon (\theta)$ and the ground-truth $y$, 
\begin{equation}
\label{EQ:composite_loss}
\ell(y^*_\varepsilon(\theta), y)
= \ell(y^*_\varepsilon(g_w(x)), y).
\end{equation}
For first-order methods, differentiating the above w.r.t. $w$  requires not only the usual model-dependent Jacobian $J_w g_w(x)$, but also a gradient in the first argument of the loss $\ell$. If this block is a strict discrete maximizer $y^*$, as noted above, the computational graph is broken. However, with our proposed modification, we have that the gradient of Eq.~\eqref{EQ:composite_loss} w.r.t. $\theta$ is equal to
\begin{equation}
\label{eq:anyloss}
J_\theta \,  y^*_\varepsilon(\theta) \, \nabla \ell(y^*_\varepsilon(\theta), y)\, ,
\end{equation}
where $\nabla \ell$ is the gradient w.r.t. the first argument of $\ell$.
Thus, the gradient can be fully backpropagated. Perturbed maximizers can therefore be used in end-to-end prediction models, for any loss $\ell$ on the perturbed maximizer. Furthermore, we describe in Section~\ref{SEC:FY} a loss that can be directly optimized in $\theta$ by first-order methods. It comes with a strong algorithmic advantage, as it requires only to compute the perturbed maximizer and not its Jacobian.
\end{rem}

\paragraph{Practical implementation.}

 For any $\theta$,  the perturbed maximizer $y^*_\varepsilon(\theta)$ is a solution of a convex optimization problem in Eq.~(\ref{EQ:regul}), allowing computation if $\Omega$ has a simple form. More generally, by their expressions as expectations, the perturbed maximizer and its Jacobian can be approximated with Monte-Carlo methods. This only requires to efficiently sample from $\mu$, and to solve LPs over $\cC$.

\begin{defn} Given $\theta\in\R^d$, let $(Z^{(1)},\ldots,Z^{(M)})$ be $M$ i.i.d. copies of $Z$ and, for $m=1,\ldots,M$,
 \[ \textstyle
 y^{(m)} = y^*(\theta + \varepsilon Z^{(m)}) = \argmax_{y \in \cC} \langle y , \theta+ \varepsilon Z^{(m)}\rangle\, .
 \]
A \emph{Monte-Carlo estimate} $\bar y_{\varepsilon,M}(\theta)$ of $y_\varepsilon^*(\theta)$ is given by 
\[\bar y_{\varepsilon,M}(\theta) = \frac 1M \sum_{m = 1}^M y^{(m)}\, .\]
\end{defn} 
Since $\E[y^{(m)}] =  y_\varepsilon^*(\theta)$ for every $m \in \{1,\dots,M\}$, by definition of $p_\theta$, it is an unbiased estimate of $y_\varepsilon^*(\theta)$. Note that the formulae in Proposition~\ref{PRO:ipp} give several manners to stochastically approximate $F_\varepsilon$, $y^*_\varepsilon$, and their derivatives by using $F(\theta+\varepsilon Z^{(m)})$, $y^*(\theta+\varepsilon Z^{(m)})$ and $\nabla_z \nu(Z^{(m)})$ and averages. This yields unbiased estimates for $F_\varepsilon$, $y^*_\varepsilon$, and its Jacobian. The plurality of these formulae gives the user several options for practical implementation. For both $y^*_\varepsilon$ and its Jacobian, we use the first one presented in Proposition~\ref{PRO:ipp} for our applications.

A great strength of this method is the absence of conceptual or computational overhead. Further, even though our analysis relies on the specific structure of the problem as an LP, these algorithms do not. The Monte-Carlo estimates can be obtained by using a function $y^*$ as a blackbox, without requiring knowledge of the problem or of the algorithm that solves it. For instance, for ranking, solving the LP only involves a sort.

If $y^*_\varepsilon$ or its derivatives are used in stochastic gradient descent for training in supervised learning, a full approximation of the gradients is not always necessary. Taking only $M=1$ (or a small number) of observations is acceptable here, as the gradients are stochastic in the first place.
 
 With parallelization and warm starts, we can alleviate the dependency in $M$ of the running time: We can independently sample the $Z^{(m)}$ and compute the $y^{(m)} = y^*(\theta + \varepsilon Z^{(m)})$ in parallel. On the other hand, starting from a solution or near-solution (such as $y^*(\theta)$) as initialization can improve running times dramatically, especially at lower temperatures.

\section{Perturbed model learning with Fenchel-Young losses}
\label{SEC:FY}

There is a large literature on learning parameters of a Gibbs distribution based on data $(y_i)_{i=1,\ldots,n}$, through maximization of the likelihood:
\begin{equation}
\label{EQ:gibbsMLE}
 \textstyle   \bar \ell_n(\theta) \!=\! \frac 1n \sum_{i=1}^n \log p_{{\sf Gibbs}, \theta}(y_i) \!=\! \frac{1}{n} \sum_{i=1}^n \langle y_i , \theta \rangle - \log Z(\theta)\; \text{with} \; \nabla_\theta \bar \ell_n(\theta)\! = \! \frac{1}{n} \sum_{i=1}^n y_i - \E_{{\sf Gibbs}, \theta}[Y].
\end{equation}
The expression of the gradient justifies the name of moment-matching procedures. The expectation of the Gibbs is however hard to evaluate in some cases.
For instance, for permutation problems, it is known to be \#P-hard to compute \citep{valiant1979complexity,taskar-thesis}.
This motivates its replacement by $p_\theta$ (perturb-and-MAP in this literature), and to use this method as a proxy for log-likelihood to learn the parameters \citep{papandreou_2011}.

 We show here that this approach can be formally analyzed by the use of Fenchel-Young losses \citep{blondel_2019} in this context. It is equivalent to maximizing a term akin to Eq.~(\ref{EQ:gibbsMLE}), substituting the log-partition $Z(\theta)$ with $F_\varepsilon(\theta)$.  The use of these losses also drastically improves the algorithmic aspects of the learning tasks, by the specific expression of the gradients of the loss. 
\begin{defn}
\label{definition:FY_loss}
In the perturbed model, the {\em Fenchel-Young} loss $L_\varepsilon(\cdot \, ; y)$ is defined for $\theta \in \R^d$ by
\[
L_\varepsilon(\theta \, ; y) = F_\varepsilon(\theta) + \varepsilon \, \Omega(y) - \langle \theta , y \rangle\, .
\]
\end{defn}

It is nonnegative, convex in $\theta$, and minimized with value 0 if and only if $\theta$ is such that $y^*_\varepsilon(\theta) = y$. It is equal to the Bregman divergence associated to $\varepsilon \Omega$, i.e., $L_\varepsilon(\theta \, ; y) = D_{\varepsilon \Omega}(y, \hat y^*_\varepsilon(\theta))$. As $\theta$ and $y$ interact in this loss only through a scalar product, for random $Y$ we have $\E[L_\varepsilon(\theta ;Y)] = L_\varepsilon(\theta ; \E[Y]) + C$, where $C$ does not depend on $\theta$. This is particularly convenient in analyzing the performance of Fenchel-Young losses in generative models.
The gradient of the loss is
\[
\nabla_\theta L_\varepsilon(\theta \, ; y) = \nabla_\theta F_\varepsilon(\theta) - y =  y^*_\varepsilon(\theta) - y\, .
\]
The Fenchel-Young loss can therefore be interpreted as a loss in $\theta$ that is a function of $y^*_\varepsilon (\theta)$. Moreover, it can be optimized in $\theta$ with first-order methods simply by computing the soft maximizer, without having to compute its Jacobian. It is therefore a particular case of the situation described in Eq.(\ref{EQ:composite_loss}) and~(\ref{eq:anyloss}), allowing to even bypass virtually the perturbed maximizer block in the output, and to directly optimize a loss between observation $y$ and model outputs $\theta = g_w(x)$.


\paragraph{Supervised and unsupervised learning.}

As described in Remark~\ref{REM:learning}, given observations $(x_i,y_i)_{1\le i \le n}\in \cX^n \times \cY^n$, we can fit a model $g_w$ such that $y_\varepsilon^*(g_w(x_i)) \approx y_i$. The Fenchel-Young loss between $g_w(x_i)$ and $y_i$ is a natural way to do so
\[
\label{EQ:genmodel}
L_{\varepsilon, {\sf emp}}(w) = \frac 1n \sum_{i=1}^n L_\varepsilon(g_w(x_i) \, ; y_i) \, ,\;\]

This is motivated by a generative model where, for some $w_0$
\[
y_i = \argmax_{y \in \cC}\ \langle g_{w_0}(x_i) + \varepsilon Z^{(i)},y \rangle \, .
\]
Indeed, under this model the {\em population loss} $\E[L_{\varepsilon, {\sf emp}}(w)]$ is the average of terms $L_\varepsilon(g_w(x_i) \, ;  y^*_\varepsilon(g_{w_0}(x_i)))$, up to an additive constant. The population loss is therefore minimized at $w_0$. The gradient of the empirical loss is given by
\begin{equation*}
\nabla_w L_{\varepsilon, {\sf emp}}(w) = \frac 1n \sum_{i=1}^n J_w \, g_w(x_i) \cdot  (y^*_\varepsilon(g_w(x_i))- y_i)\, .
\end{equation*}
Each term in the sum, gradient of the loss for a single observation, is therefore a stochastic gradient for $L_{\varepsilon, \sf emp}$ (w.r.t. $i$ uniform in $[n]$) or for $L_{\varepsilon, \sf pop}$ (w.r.t. to a random $y_i$ from $p_{g_{w_0}(x_i)}$). 

The methods we described to stochastically approximate the gradient are particularly adapted here. Indeed, following~\citep{shpakova_2016}, given an observation $y_i$ and a current value $\theta_i = g_w(x_i)$, a {\em doubly stochastic} version of the gradient $\nabla_w L_\varepsilon(g_w(x_i) \, ; y_i)$ is obtained by
\begin{equation}
\label{EQ:doublystoch}
   \bar \gamma_{i,M}(w) = J_w \, g_w(x_i) \big(\frac{1}{M} \sum_{m = 1}^M y^*\big(g_w(x_i)+\varepsilon Z^{(m)}\big) - y_i\big)\, .
\end{equation}
This can also be used with a procedure where batches of data points are used to compute approximate gradients, where the number of artificial samples $M$ and the batch size can be chosen separately.

This can  be extended to an unsupervised setting, where observations $(y_i)_{1\le i \le n}\in  \cY^n$ are fitted with a model $p_\theta$, motivated by a generative model where $y_i = \argmax_{y \in \cC}\  \langle \theta_0 + \varepsilon Z_i ,y \rangle$, that is $y_i \sim p_{\theta_0}(y)$, for some unknown $\theta_0$. We have a natural {\em empirical} $\bar L_n$ and {\em population} loss $L_{\theta_0}$:
\[
 \bar L_{\varepsilon, n}(\theta) \!=\! \frac 1n \sum_{i=1}^n L_\varepsilon(\theta \, ; y_i) \!= \! L_\varepsilon(\theta ; \bar Y_n) + C(Y) \,, L_{\varepsilon, \theta_0}(\theta) \!= \! \E[\bar L_{\varepsilon, n}(\theta)] \!= \! L_\varepsilon(\theta ; y^*_\varepsilon(\theta_0)) + C(\theta_0)\, .
\]
Their gradients are given by
\[
\nabla_\theta \bar L_{\varepsilon, n}(\theta) = \nabla_\theta F_\varepsilon(\theta) - \bar Y_n =  y^*_\varepsilon(\theta) - \bar Y_n \,,\quad \text{and} \quad
\nabla_\theta L_{\varepsilon, \theta_0}(\theta) = y^*_\varepsilon(\theta) - y^*_\varepsilon(\theta_0)\, .
\]
The empirical loss is minimized for $\hat \theta_n$ such that $y^*_\varepsilon(\hat \theta_n) = \bar Y_n$ and the population loss when $y^*_\varepsilon(\theta) = y^*_\varepsilon(\theta_0)$. As a consequence, the whole battery of statistical results, from asymptotic to non-asymptotic, can be leveraged, and we present the simplest one (asymptotic normality).
\begin{prop}
\label{PRO:asymptNorm}
When $n$ goes to $\infty$, with the assumptions of Proposition~\ref{PRO:fundPROP} on the model, we have
\[
\sqrt{n}(\hat \theta_n - \theta_0) \to \cN\big(0,\big(\nabla_\theta^2 F_\varepsilon(\theta_0)\big)^{-1}\Sigma_Y \big(\nabla_\theta^2 F_\varepsilon(\theta_0)\big)^{-1}\big)\, ,
\]
in distribution, where $\Sigma_Y$ is the covariance of $Y\sim p_\theta$.
\end{prop}

\section{Experiments}
\label{SEC:expe}
We demonstrate the usefulness of perturbed maximizers in a supervised learning setting, as described in Section~\ref{SEC:FY}. 
We focus on a classification task and on two structured prediction tasks, label ranking and learning to predict shortest paths. 
Since we focus on the prediction task, the issues raised in Remark~\ref{REM:empty} do not apply. 
When learning with the Fenchel-Young losses, we simulate doubly stochastic gradients $\nabla_w L_\varepsilon(g_w(x_i) \, ; y_i)$ of the empirical loss with $M$ artificial perturbations (see Equation~\ref{EQ:doublystoch}).

We will open-source a Python package allowing to turn any black-box solver into a differentiable function, in just a few lines of code. Full details of the experiments are included in Appendix~\ref{APP:expe}.

\subsection{Perturbed max}
\label{SEC:expemax}
We use the perturbed argmax with Gaussian noise in an image classification task on the CIFAR-10 dataset. This serves two purposes: showing that we perform as well as the cross entropy loss, in a case where a soft max can be easily computed, and exhibiting the impact of the algorithmic parameters. We train a vanilla-CNN 
with 10 network outputs that are the entries of $\theta$, we minimize the Fenchel-Young loss between $\theta_i = g_w(x_i)$ and $y_i$, with different temperatures $\varepsilon$ and number of perturbations $M$. We observe competitive performance compared to standard losses as baselines (Fig.~\ref{FIG:classif}, left and center).

\begin{figure*}[ht!]
\begin{center}
\includegraphics[width=\textwidth]{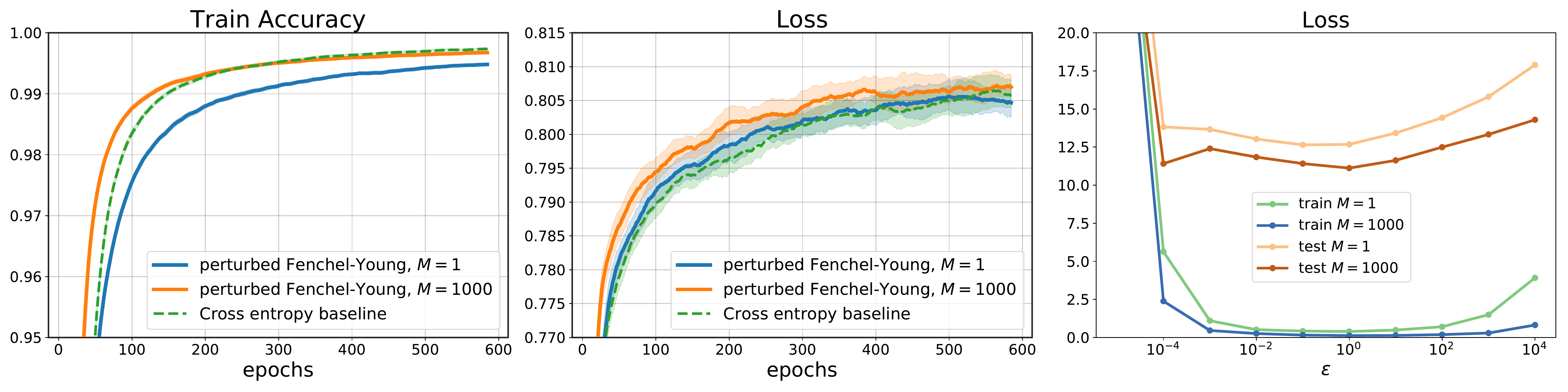}
\end{center}
\caption{{\bf Left.} Accuracy in training, using the perturbed FY loss, or cross entropy baseline. {\bf Center.} Test accuracy for these methods. {\bf Right.} Impact of the parameter $\varepsilon$ on test and train squared loss.} 
\label{FIG:classif}
\end{figure*}
We analyze the impact of the algorithmic parameters on optimization and generalization abilities. We exhibit the final loss and accuracy for different number of perturbations in the doubly stochastic gradient ($M=1,1000$). We highlight the importance of the temperature parameter $\varepsilon$ on the algorithm (see Figure~\ref{FIG:classif}, right). Very high or low temperatures degrade the ability to fit to training and to generalize to test data, by lack of smoothing or loss of information about $\theta$. We also observe that our framework is very robust to the choice of $\varepsilon$, demonstrating its adaptivity.


\begin{figure}[H]
\begin{center}
\includegraphics[width = 0.7\textwidth]{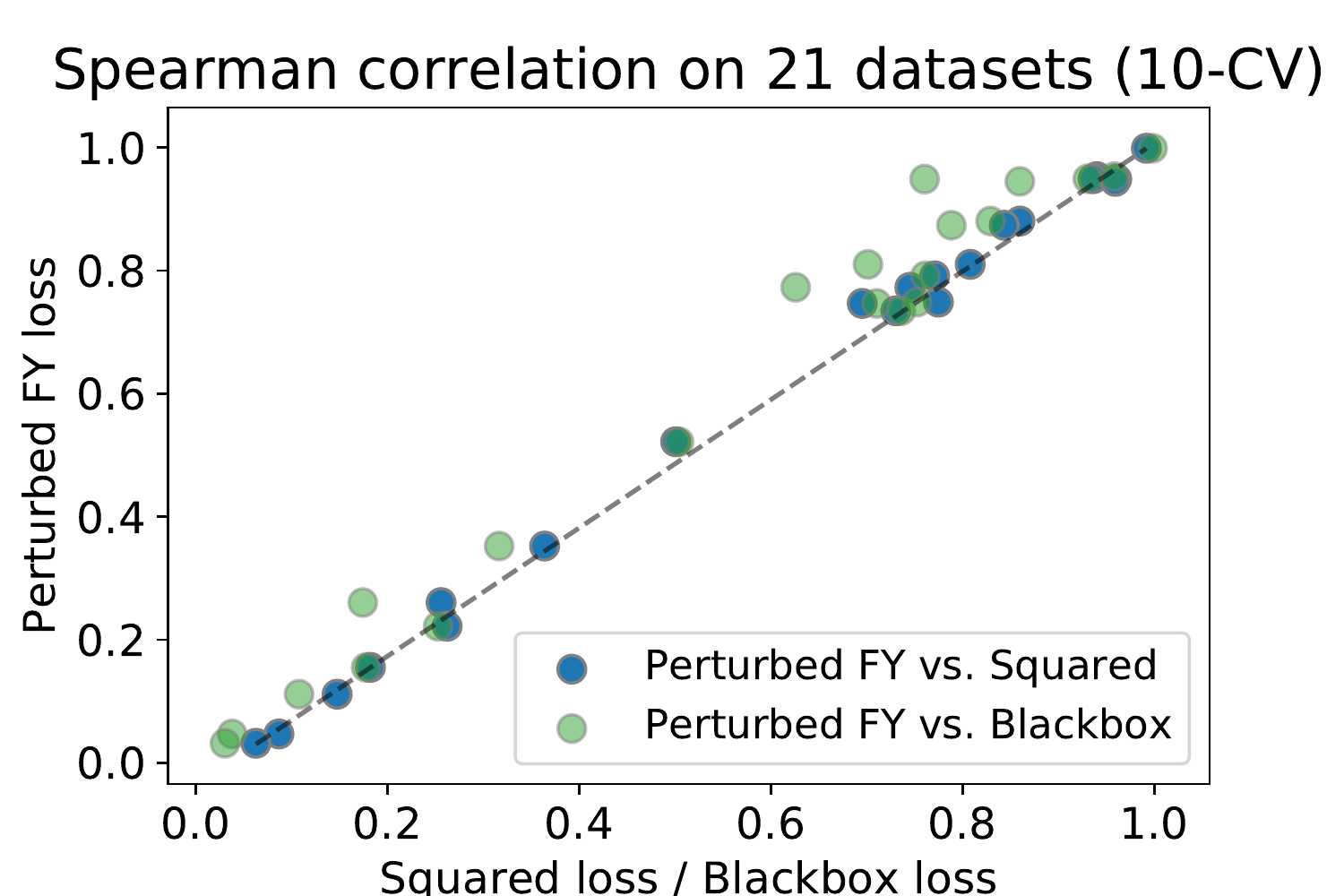}
\end{center}
\caption{Comparison on 21 datasets, for our proposed perturbed Fenchel-Young loss, a squared loss and the blackbox loss of \cite{vlastelica2019differentiation}. Points above the diagonal are datasets where our loss performs better.\label{FIG:ranking}
}
\end{figure}
\subsection{Perturbed label ranking}

We consider label ranking tasks, where each $y_i$ is a label permutation for features $x_i$. We minimize the weights of an affine model $g_w$ (i.e., $\theta_i = g_w(x_i)$) using our perturbed Fenchel-Young loss, a simple squared loss and the recently-proposed blackbox loss of \citet{vlastelica2019differentiation}. 
Note that our loss is convex in $\theta$ and enjoys unbiased gradients, while \citep{vlastelica2019differentiation} uses a non-convex loss with gradient proxies.
We use the same 21 datasets as in \citep{hullermeier2008label, cheng2009decision}.
We report Spearman's correlation (higher is better) in Figure~\ref{FIG:ranking}. Results are averaged over 10-fold CV and parameters tuned by 5-fold CV. We find that  our loss performs better or similarly (within a $5$\% range) on $76$ \% and $90$ \% of the datasets, respectively. 
Detailed experimental setup and  results are given in Appendix \ref{APP:ranking}.


To better understand the complexity of this task, we also created a range of artificial datasets where 100 labels are generated by $y_i = \argmax_y \langle x_i^\top w_0 + \sigma Z_i,y \rangle$, in dimension 50, for different values of $\sigma$. We minimize the same losses as before in $w$. For almost correct labels ($\sigma \approx 0$), our method accurately generalizes to the test data (see Figure~\ref{FIG:ranking}, and Figure~\ref{fig:partial-artificial} in Appendix~\ref{APP:expe} for other metrics). We observe that the Fenchel-Young loss performs as well or better than the other losses, particularly

\begin{figure}[H]
\begin{center}
\includegraphics[width = 0.8\textwidth]{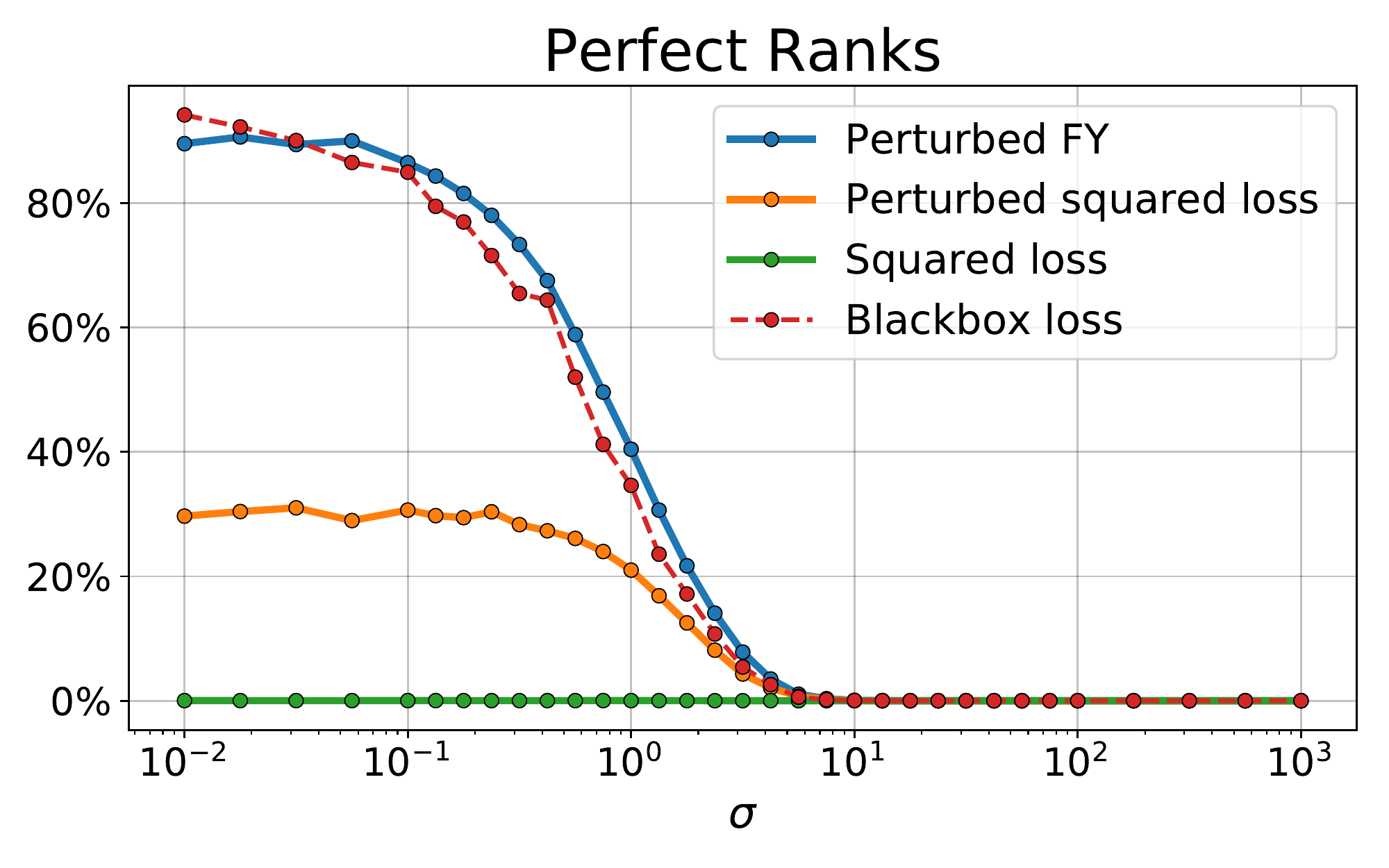}
\end{center}
\caption{Average number of instances with exactly correct ranks for all 100 labels, for different values of $\sigma$, for four methods.\label{FIG:artificial}
}
\end{figure}
 in terms of robustness to the noise. All details are included in Appendix~\ref{APP:ranking}.



\subsection{Perturbed shortest path}

We replicate the experiment of \citet{vlastelica2019differentiation}, aiming to learn the travel costs in graphs based on features, given examples of shortest path solutions (see Figure~\ref{FIG:warcraft}). We use a dataset of 10,000 RGB images of size $96 \times 96$ illustrating Warcraft terrains of $12\times 12$ 2D grid networks. The responses $y_i$ are a shortest path between the top-left and bottom-right corners, for costs hidden to the network, corresponding to the terrain type. They are $12\times12$ binary matrices representing the \mbox{vertices along the shortest path}.
\begin{figure}[H]
\begin{center}
\includegraphics[width=\textwidth]{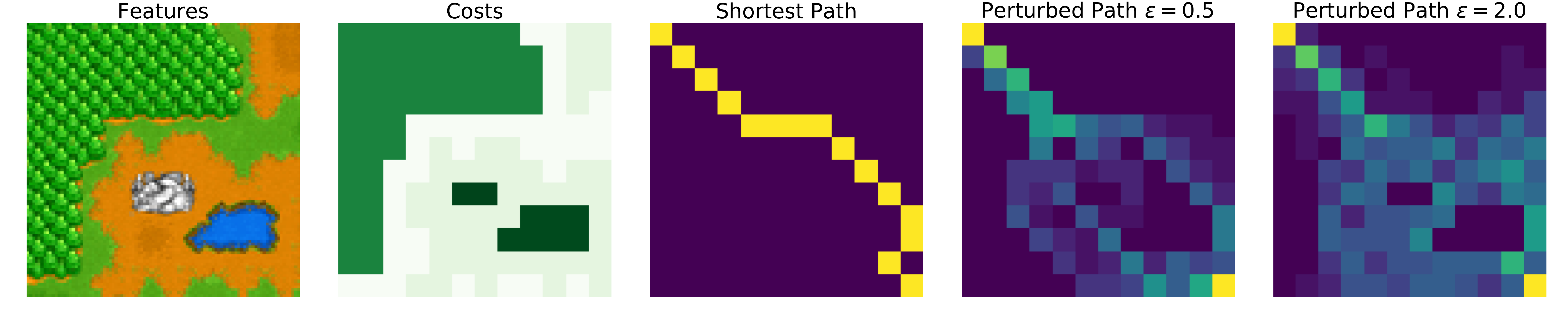}
\end{center}


\caption{In the shortest path experiment, training features are images. Shortest paths are computed based on terrain costs, hidden to the network. Training responses are shortest paths based on this cost. \label{FIG:warcraft}}
\end{figure}

Following \cite{vlastelica2019differentiation}, we train a network whose first five layers are those of ResNet18 for the Fenchel-Young loss between the predicted costs $\theta_i = g_w(x_i)$ and the shortest path $y_i$. We optimize over $50$ epochs with batches of size $70$, temperature $\varepsilon=1$ and $M=1$ (single perturbation). We are able, only after a few epochs, to generalize very well, and to accurately predict the shortest path on the test data. We compare our method to two baselines, from \citep{vlastelica2019differentiation}: training the same network with their proposed blackbox loss and with a squared loss. We show two metrics: perfect accuracy percentage and cost ratio to optimal path (see Figure~\ref{FIG:shortestpathperf}); full implementation details are in  Appendix~\ref{APP:short}.
\begin{figure*}[ht!]
\begin{center}
\includegraphics[width=\textwidth]{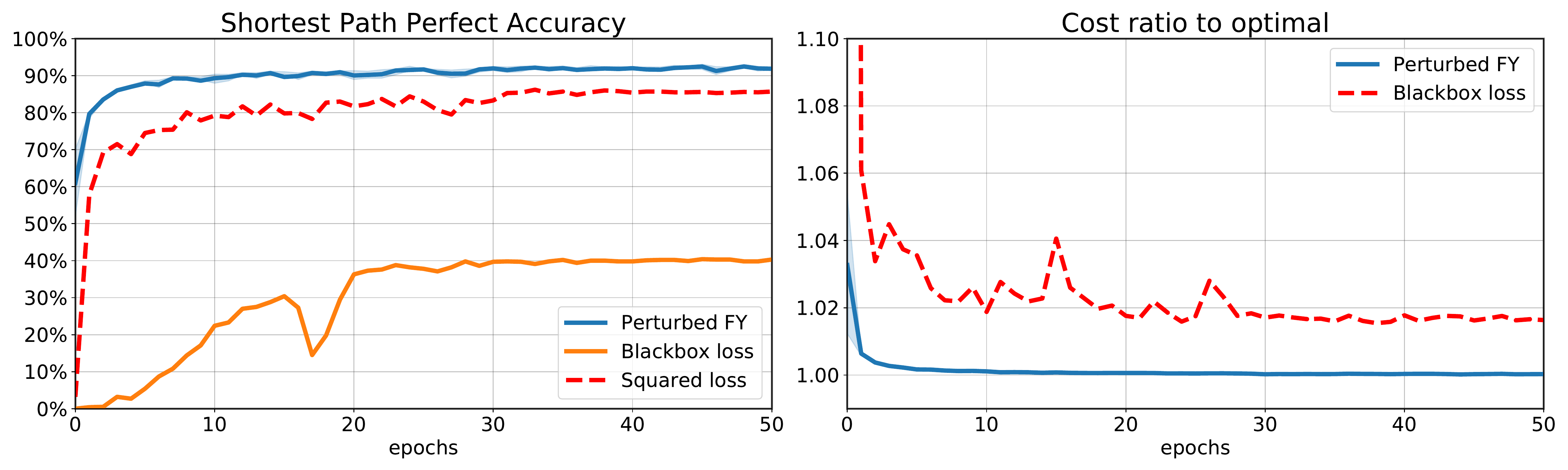}

\vspace{-0.1cm}
\caption{Accuracy of the predicted path for several methods during training. {\bf Left.} Percentage of test instances where the predicted path is optimal. {\bf Right.} Ratio of costs between the predicted path and the actual shortest path -- without the squared loss baseline as it does not yield valid paths. \label{FIG:shortestpathperf}}
\end{center}
\end{figure*}
\section{Conclusion}

Despite a large body of work on perturbations techniques for machine learning,
most existing works focused on approximating sampling, log-partitions and expectations under
the Gibbs distribution. Together with novel theoretical insights, we propose to use a general perturbation framework to differentiate through, not only a max, but also an argmax,
without ad-hoc modification of the underlying solver. In addition, by defining an equivalent
regularizer $\Omega$, we show how to construct Fenchel-Young losses and propose a doubly 
stochastic scheme, enabling learning in various tasks, and validate on experiments its ease of application.


\paragraph{Aknowledgements.} FB's work was funded in part by the French government under management of Agence Nationale de la Recherche as part of the ``Investissements d'avenir'' program, reference ANR-19-P3IA-0001 (PRAIRIE 3IA Institute). FB also acknowledges  support from the European Research Council (grant SEQUOIA 724063).


\bibliography{references}
\bibliographystyle{apalike}

\clearpage

\appendix

\section{Proofs of technical results}
\label{SEC:appProofs}

\begin{proof}[Proof of Proposition~\ref{PRO:duality}]
The function $\varepsilon \Omega$ is the Fenchel dual of $F_\varepsilon$ (see Proposition~\ref{PRO:fundPROP}, impact of the temperature), and is defined on $\cC$. As such, as in \cite{AbeLeeSin14}, we have that
\[
F_\varepsilon(\theta) = \sup_{y \in \cC}\big\{\langle \theta,y \rangle - \varepsilon \Omega(y) \big\}\, .
\]
It is maximized at $\nabla_\theta F_\varepsilon(\theta) = y_\varepsilon^*(\theta)$, by Fenchel-Rockaffelar duality \citep[see, e.g.][Appendix A]{wainwright2008graphical}.
\end{proof}

\begin{proof}[Proof of Proposition~\ref{PRO:fundPROP}]

The proof of these properties makes use of the notion of the {\em normal fan} of $\cC$. It is the set of all {\em normal cones} to all faces of the polytope $\cC$ \citep{rockafellar2009variational}. For each face, such a cone is the set of vectors in $\R^d$ such that the linear program on $\cC$ with this vector as cost is maximized on this face. They form a partition of $\R^d$, and these cones are full dimensional if and only if they are associated to a vertex  of $\cC$. These vertices are a subset $\cE$ of $\cY$, corresponding to extreme points of $\cC$.

As a consequence of this normal cone structure, since $\mu$ has a positive density, it assigns positive mass to sets if and only if they have non-empty interior, so for any $\theta \in \R^d$, and any $\varepsilon>0$, $p_\theta(y) > 0$ if and only if $y \in \cE$. In most applications, $\cE=\cY$ to begin with (all $y$ are potential maximizer for some vector of costs, otherwise they are not included in the set), and all points in $\cY$ have positive mass.

{\bf  Properties of $F_\varepsilon$}

- $F_\varepsilon$ is strictly convex

The function $F$ is convex, as a maximum of convex (linear) functions. By definition of $F_\varepsilon$, for every $\lambda \in [0,1]$ and $\theta, \theta' \in \R^d$, for $\theta_\lambda = \lambda \theta + (1-\lambda) \theta'$ we have
\[
\lambda F_\varepsilon(\theta) + (1-\lambda) F_\varepsilon(\theta') = \E[\lambda F(\theta+\varepsilon Z) + (1-\lambda) F(\theta'+ \varepsilon Z)] \le \E[ F(\lambda \theta + (1-\lambda) \theta'+\varepsilon Z)]
 = F_\varepsilon(\theta_\lambda)\, .
\]
The inequality holds with equality if and only if it holds within the expectation for almost all $z$ since the distribution of $Z$ is positive on $\R^d$. If the function $F_\varepsilon$ is not strictly convex, there exists therefore $\theta$ and $\theta'$ such that 
\[
\lambda F(\theta+\varepsilon z) + (1-\lambda) F(\theta'+ \varepsilon z) = F(\lambda \theta + (1-\lambda) \theta'+\varepsilon z)
\]
for all $\lambda \in [0,1]$, for almost all $z \in \R^d$. In this case, $F$ is linear on the segment $[\theta + \varepsilon z, \theta'+\varepsilon z]$ for almost all $z \in \R^d$.

If $\theta - \theta'$ is contained in the boundary between the normal cones to $y_1$ and $y_2$, for all distinct $y_1,y_2 \in \cE$, we have $\langle y_1 - y_2 , \theta-\theta'\rangle = 0$ for all such pairs of $y$, so $\theta$ is orthogonal to the span of all the pairwise differences of $y$. However, since $\cC$ has no empty interior, it is not contained in a strict affine subspace of $\R^d$ so $\theta - \theta' =0$. As a consequence, for distinct $\theta$ and $\theta'$, there exists $z \in \R^d$ such that $\theta + \varepsilon z$ and $\theta + \varepsilon z$ are in the interior of two normal cones to different $y\in \cE$. As a consequence, the same holds under perturbations of $z$ in a small enough ball of $\R^d$, so $F$ cannot be linear on almost all segments $[\theta + \varepsilon z, \theta'+\varepsilon z]$, and $F_\varepsilon$ is strictly convex.

- $F_\varepsilon$ is twice differentiable, as a direct consequence of Proposition~\ref{PRO:ipp}.

- $F_\varepsilon$ is $R_\cC$-Lipschitz

$F$ is the maximum of finitely many functions that are $R_\cC$-Lipschitz. It therefore also satisfies this property. $F_\varepsilon$ is an expectation of such functions, therefore it satisfies the same property.

- $F_\varepsilon$ is $R_\cC M_\mu / \varepsilon$-gradient Lipschitz.

We have, by Proposition~\ref{PRO:ipp}, for $\theta$ and $\theta'$ in $\R^d$
\[
\nabla_\theta F_\varepsilon(\theta) - \nabla_\theta F_\varepsilon(\theta') = \E[(F(\theta+\varepsilon Z) -  F(\theta'+\varepsilon Z))\nabla_z \nu(Z)/\varepsilon]\, .
\]
As a consequence, by the Cauchy--Schwarz inequality, and Lipschitz property of $F$, it holds that
\begin{align*}
\| \nabla_\theta F_\varepsilon(\theta) - \nabla_\theta F_\varepsilon(\theta') \| &\le \E[\|F(\theta+\varepsilon Z) -  F(\theta'+\varepsilon Z)\|^2]^{1/2} \E[\|\nabla_z \nu(Z)\|^2 /\varepsilon^2]^{1/2}\\
&\le R_\cC\|\theta - \theta' \| \E[\|\nabla_z \nu(Z)\|^2 ]^{1/2}/\varepsilon = (R_\cC M_\mu /\varepsilon)\|\theta - \theta' \| \, .
\end{align*}

{\bf Properties of $\Omega$}

The function $\varepsilon \Omega$ is the Fenchel dual of $F_\varepsilon$, which is strictly convex and $R_\cC M_\mu/\varepsilon$ smooth. As a consequence, $\Omega$ is differentiable on the image of $y^*_\varepsilon$ -- the interior of $\cC$ -- and it is $1/R_\cC M_\mu$-strongly convex.

- Legendre type property

The regularization function $\Omega$ is differentiable on the interior. If there is a point $y$ of its boundary such that $\nabla_y \Omega$ does not diverge when approaching $y$, then taking $\theta$ such that $\theta - \varepsilon \nabla_y \Omega(y) \in \cN_\cC(y)$ (where $\cN_\cC(y)$ is the normal cone to $\cC$ at $y$), then $y_\varepsilon^*(\theta) = y$. However, $y^*_\varepsilon$ takes image in the interior of $\cC$ (see immediately below), leading to a contradiction.

{\bf Properties of $y^*_\varepsilon$}

- The perturbed maximizer is in the interior of $\cC$

Since the distribution of $Z$ has positive density, the probability that $\theta + \varepsilon Z \in \cN_\cC(y)$ (i.e. $p_\theta(y)$) is positive for all $y \in \cE$. As a consequence, since
\[
y^*_\varepsilon(\theta) = \sum_{y \in \cE} y \, p_\theta(y)\, ,
\]
with all positive weights $p_\theta(y)$, $y^*_\varepsilon$ is in the interior of the convex hull $\cC$ of $\cE$.

- The function $y^*_\varepsilon$ is differentiable, by twice differentiability of $F_\varepsilon$, by Proposition~\ref{PRO:ipp}.

{\bf Influence of temperature parameter $\varepsilon >0$}

We have for all $\theta$
\[
F_\varepsilon(\theta) = \E[\max_{y \in \cC} \langle y ,\theta + \varepsilon Z \rangle] = \varepsilon \E[\max_{y \in \cC} \langle y/\varepsilon ,\theta + Z \rangle] = \varepsilon F_1(\theta/\varepsilon)\, .
\]
As a consequence
\[
(F_\varepsilon)^*(y) = \max_{z \in \R^d} \big\{\langle y,z\rangle - F_\varepsilon(z) \big\} = \varepsilon \max_{z \in \R^d} \big\{\langle y,z/\varepsilon \rangle - \varepsilon F_1(z/\varepsilon) \big\} = \varepsilon (F_1)^*(y) = \varepsilon \, \Omega(y)\, .
\]
Since $y^*_\varepsilon(\theta) = \nabla_\theta F_\varepsilon(\theta)$, and since $F_\varepsilon(\theta) = \varepsilon F_1(\theta/\varepsilon)$, we have $y^*_\varepsilon(\theta) = y_1^*(\theta/\varepsilon)$.
\end{proof}

\begin{proof}[Proof of Proposition~\ref{PRO:temperature}]
We recall that we assume that $\theta$ yields a unique maximum to the linear program on $\cC$. This is true almost everywhere, and assumed here for simplicity of the results. We discuss briefly at the end of this proof how this can be painlessly extended to the more general case.

{\bf Limit at low temperatures $(\varepsilon \to 0)$}

Since $F$ is convex (see proof of Proposition~\ref{PRO:fundPROP}), so by Jensen's inequality
\begin{align*}
F\big(\E[\theta+\varepsilon Z] \big) &\le \E\big[F(\theta+\varepsilon Z) \big]\\
F(\theta) &\le F_\varepsilon(\theta)\, .
\end{align*}
Further, we have for all $Z \in \R^d$
\[
\max_{y \in \cC}\langle \theta + \varepsilon Z ,y \rangle \le \max_{y \in \cC}\langle \theta ,y \rangle +\varepsilon \max_{y' \in \cC}\langle  Z ,y' \rangle
\]
Taking expectations on both sides yields that
\[
F_\varepsilon(\theta) \le F(\theta) + \varepsilon F_1(\theta)\, .
\]
As a consequence, when $\varepsilon \to 0$, combining these two inequalities yields that $F_\varepsilon(\theta) \to F(\theta)$.

Regarding the behavior of the perturbed maximizer $y^*_\varepsilon(\theta)$, we follow the arguments of \citep[][Proposition 4.1]{peyre2019computational}. By Proposition~\ref{PRO:duality} and the definition of $y^*(\theta)$, we have
\[
0 \le \langle y^*(\theta), \theta \rangle - \langle y^*_\varepsilon(\theta),\theta\rangle \le  \varepsilon\big[\Omega\big(y^*(\theta)\big) - \Omega\big(y^*_\varepsilon(\theta)\big)\big]
\]
Since $\Omega$ is continuous, it is bounded on $\cC$, and the right hand term above is bounded by $C \varepsilon$, for some $\varepsilon>0$. As a consequence, when $\varepsilon \to 0$, $\langle y^*_\varepsilon(\theta),\theta \rangle \to \langle y^*(\theta),\theta \rangle$. For any sequence $\varepsilon_n \to 0$, the sequence $y_n = y^*_{\varepsilon_n}(\theta)$ is in a compact $\cC$. Therefore, it has a subsequence $y_{\varphi(n)}$ that converges to some limit $y_\infty \in \cC$. However, since $\langle y^*_{\varphi(n)},\theta \rangle \to \langle y^*(\theta),\theta \rangle$, we have $\langle y_\infty,\theta \rangle = \langle y^*(\theta),\theta \rangle$, by continuity. Since $y^*(\theta)$ is a unique maximizer, $y_\infty = y^*(\theta)$. As a consequence, all convergent subsequences of $y_n$ converge to the same limit $y^*(\theta)$: it is the unique accumulation point of this sequence. It follows directly that $y_n$ converges to $y^*(\theta)$, as it lives in a compact set, which yields the desired result.

{\bf Limit at high temperatures} By Proposition~\ref{PRO:fundPROP}, $y^*_\varepsilon(\theta) = y_1(\theta/\varepsilon)$, so the desired result follows by continuity of the perturbed maximizer.

{\bf Nonasymptotic inequalities.} These inequalities follow directly from those proved to establish limits at low temperatures.

If $\theta$ is such that the maximizer is not unique (which occurs only on a set of measure 0), the only result affected is the convergence of $y^*_\varepsilon(\theta)$ when $\theta \to 0$. Following the same proof of \citep[][Proposition 4.1]{peyre2019computational}, it can be shown to converge to the minimizer of $\Omega$ over the set of maximizer. This point is always unique, as the minimizer of a strongly convex function over a convex set.
\end{proof}

\begin{proof}[Proof of Proposition~\ref{PRO:asymptNorm}]
We follow the classical proofs in M-estimation \citep[see, e.g.][Section 5.3]{van2000asymptotic}. First, the estimator is consistent as a virtue of the continuous mirror map between $\R^d$ and $\text{int}(\cC)$. For $n$ large enough $\bar Y_n \in \text{int}(\cC)$, since the probability of each extreme point of $\cC$ is positive. By definition of the estimator and stationarity condition for $\hat \theta_n$, we have in these conditions
\[
\nabla_\theta F_\varepsilon(\hat \theta_n) = \bar Y_n, \; \nabla_\theta F_\varepsilon(\theta_0) = y^*_\varepsilon(\theta_0)\, .
\]
By the law of large numbers, $\bar Y_n$ converges to its expectation $y^*_\varepsilon(\theta_0)$ a.s. Since $\varepsilon \nabla_y \Omega$, the inverse of $\nabla_\theta F_\varepsilon$, is also continuous (by the fact that $\Omega$ is convex smooth), we have that $\hat \theta_n$ converges to $\theta_0$ a.s.

We write the first order conditions for $\bar L_{\varepsilon,n}$ at $\hat \theta_n$ and the Taylor expansion with Lagrange remainder for all coordinates, one by one
\begin{equation}
\label{EQ:taylorNormal}
0 = \nabla_\theta \bar L_{\varepsilon,n}(\hat \theta_n) = \nabla_\theta \bar L_{\varepsilon,n}(\theta_0) + A_n (\hat \theta_n - \theta_0)\,,
\end{equation}
where $A$ is such that, for all coordinates $i \in [d]$
\[
A_i = (\nabla^2_\theta \bar L_{\varepsilon, n}(\bar \theta^{(i)}))_i\,
\]
for some $\bar \theta^{(i)} \in [\hat \theta_n, \theta_0]$. We note here that since the estimator is not necessarily in dimension 1, $A_n$ cannot be written directly as $\nabla^2_\theta \bar L_{\varepsilon, n}(\bar \theta)$ for some $\bar \theta \in [\hat \theta_n,\theta_0]$, since the Taylor expansion with Lagrange remainder is not true in its multivariate form. However, doing it coordinate-by-coordinate as here allows to circumvent this issue.

We have that $\nabla^2 \bar L_{\varepsilon,n} = \nabla^2 F_\varepsilon$. Since $\hat \theta_n \to \theta_0$ a.s.~ we have that $\bar \theta^{(i)} \to \theta_0$ for all $i \in [d]$, so $A_n \to \nabla^2 F_\varepsilon(\theta_0)$ a.s. Rearranging terms in Eq.~(\ref{EQ:taylorNormal}), we have 
\begin{align*}
\sqrt{n}(\hat \theta_n - \theta_0) &= -A_n^{-1} \cdot \sqrt{n}\nabla_\theta \bar L_{\varepsilon,n} (\theta_0)\\
&= A_n^{-1} \cdot \sqrt{n} \big(\bar Y_n - y^*_\varepsilon(\theta_0)\big)
\end{align*}

By the central limit theorem, $\sqrt{n}\big(\bar Y_n-y^*_\varepsilon(\theta_0)\big) \to \mathcal{N}(0,\Sigma_Y)$ in distribution. As a consequence, by convergence of $A_n$ and Slutsky's lemma, we have the convergence in distribution
\[
\sqrt{n}(\hat \theta_n - \theta_0) \to \mathcal{N}\big(0,\big(\nabla_\theta^2 F_\varepsilon(\theta_0)\big)^{-1}\Sigma_Y \big(\nabla_\theta^2 F_\varepsilon(\theta_0)\big)^{-1}\big)\, .
\]

\end{proof}

\newpage
\section{Examples of discrete decision problems as linear programs}
\label{APP:examples}
Our method applies seamlessly to all decision problems over discrete sets. Indeed, any problem of the form $\max_{y \in \cY} s(y)$, for some score function $s:\cY \to \R$, can at least be written in the form
\[
\max_{x \in \Delta^{|\cY|}} \langle x , s \rangle\, ,
\]
by representing $\cY$ as the vertices of the unit simplex in $\R^{|\cY|}$. However, for most interesting decision problem that can actually be solved in practice, the score function takes a simpler form $s(y) = \langle y , \theta \rangle$, for some representation of $y \in \R^d$ and some $\theta$. We give here a non-exhaustive list of examples of interesting problems of this type.

{\bf Maximum.} The max function from $\R^d$ to $\R$, that returns the largest among the $d$ entries of a vector $\theta$ is ubiquitous in machine learning, the hallmark of any classification task. It is equal to $F(\theta)$ over the standard unit simplex.
\[
F(\theta) = \max_{i \in [d]} \theta_i \, ,\;\cC = \{y \in \R^d\, : \, y \ge 0\, , \;\mathbf{1}^\top y = 1\}\, .
\]
On this set, using Gumbel noise yields the log-sum-exp for $F_\varepsilon$, the Gibbs distribution for $p_\theta$, and the softmax for $y^*_\varepsilon$. Using other noise distributions for $Z$ will change the model.

\noindent
{\bf Top $k$.} The function from $\R^d$ to $\R$ that returns the sum of the $k$ largest entries of a vector $\theta$ is also commonly used. It fits our framework over the set
\[
\cC = \{y \in \R^d\, : \, 0 \le y \le 1\, , \;\mathbf{1}^\top y = k\}\, .
\]
\noindent
{\bf Ranking.} The function returning the ranks (in descending order) of a vector $\theta \in \R^d$ can be written as the argmax of a linear program over the {\em permutahedron}, the convex hull of permutations of any vector $v$ with distinct entries
\[
\cC = \cP_v = \text{cvx}\{P_\sigma v \, : \, \sigma \in \Sigma_d\}\, .
\]
Using different reference vectors $v$ yield different perturbed operations, and $v=(1,2,\ldots,d)$ is commonly used.

\noindent
{\bf Shortest paths.}
For a graph $G=(V,E)$ and positive costs over edges $c \in \R^E$, the problem of finding a shortest path (i.e. with minimal total cost) from vertices $s$ to $t$ can be written in our setting with $\theta = -c$ and
\[
\cC = \{y \in \R^E : y \ge 0\,, (\mathbf{1}_{\to i} - \mathbf{1}_{i \to})^\top y = \delta_{i=s} - \delta_{i=t} \}\, .
\]

\noindent
{\bf Assignment.} The linear assignment problem, and more generally the optimal transport problem, can also be written as a linear program. In the case of the assignment problem, it is the {\em Birkhoff polytope} of doubly-stochastic matrices, whose extreme points are the permutation matrices
\[
\cC = \{Y \in \R^{d' \times d'}\, : \, Y_{ij} \ge 0, \; \mathbf{1}^\top Y = \mathbf{1}^\top,\quad Y \mathbf{1} = \mathbf{1}\}\,.
\]
There is a large literature on regularization of this problem, with entropic penalty \cite{cuturi2013sinkhorn}. This is one of the rare cases where the regularized version of the problem is actually computationally lighter, in stark contrast with the general case in our setting.

\noindent{\bf Combinatorial problems.} Many other problems, such in combinatorial optimization can be formulated exactly (e.g. minimum spanning tree, maximum flow), or approximately via convex relaxations (e.g. traveling salesman problem, knapsack), via relaxations in linear programs. Differentiable versions of these exact or approximate solutions can therefore be obtained via perturbation methods.

\noindent{\bf Relaxations with atomic norms} A wide variety of high-dimensional statistical learning problems can be tackled by regularization via atomic, or otherwise sparsity-inducing norms \cite{chandrasekaran2012convex,bach2012optimization}. Our framework also allows us to consider versions of these estimators that are differentiable in their inputs.

\newpage
\section{Experimental details}
\label{APP:expe}

\subsection{Perturbed maximum}
\label{APP:max}
In the experiment on perturbed maximum for classification on CIFAR-10, we train a vanilla-CNN made of 4 convolutional and 2 fully connected layers  for 600 epochs with batches of size 32. 

We train by minimizing two losses in the weights $w$ of the network function $g_w$, fitting the outputs $\theta_i = g_w(x_i)$ to labels $y_i$
\begin{itemize}[leftmargin=10pt]
\renewcommand \labelitemi{--}
    \item Perturbed Fenchel-Young (proposed): our proposed Fenchel-Young loss (see Definition \ref{definition:FY_loss}),
    \[
    L_\varepsilon(g_w(x_i); y_i)\, ,
    \]
    \item Cross entropy loss, for a soft max layer $s_\varepsilon$ and an entrywise $\log$
    \[
    H(g_w(x_i); y_i) = \langle y_i, \log(s_\varepsilon(g_w(x_i))) \rangle\,.
    \]
\end{itemize}

\subsection{Perturbed label ranking}
\label{APP:ranking}

In this experiment, we consider label ranking tasks, where each $y_i$ is a ground-truth label permutation for features $x_i$. We minimize the weights of an affine model $g_w$ (i.e., $\theta_i = g_w(x_i)$) using the following losses:
\begin{itemize}[leftmargin=10pt]
\renewcommand \labelitemi{--}
    \item Perturbed Fenchel-Young (proposed): our proposed Fenchel-Young loss (see Definition \ref{definition:FY_loss}),
    \item Perturbed $+$ Squared loss (proposed): $\frac{1}{2} \|y_i - y_\varepsilon^*(g_w(x_i))\|^2$, where gradients can be computed using Proposition \ref{PRO:ipp} and the chain rule,
    \item Squared loss: $\frac{1}{2} \|y_i - g_w(x_i)\|^2$,
    \item Blackbox loss: $\frac{1}{2} \|y_i - y^*(g_w(x_i))\|^2$, 
    where we use the gradient \textbf{proxy} of \citet{vlastelica2019differentiation}, re-implemented for the experiments on ranking.
\end{itemize}
We use the same 21 datasets as in \citep{hullermeier2008label, cheng2009decision}.
Detailed results are given in Table \ref{tab:ranking_results}
and Table \ref{tab:ranking_results_tuned}.

For the experiment on artificial datasets, we use the same setup as above, with a linear model instead of affine. The ground-truth vector $w_0$ is obtained by uniform sampling in $\{-1,1\}^d$, and the $x_i$ are standard isotropic normal. In the experiments presented here, we optimize over 2000 epochs, with a batch size of 32: very good results are obtained even with a smaller number of epochs, but we increased it artificially to better evaluate numerically the {\em final} predictive performance of all methods (see Figure~\ref{fig:app-epochs}). In the main text, we present in Figure~\ref{FIG:artificial} the metric of perfect rank accuracy over one run of simulations. We present in Figure~\ref{fig:partial-artificial} the same metric, as well as the metric of partial rank accuracy (i.e. the proportion of correctly ordered labels), for completeness, averaged over three runs of the dataset. To further illustrate these results, we include in Figure~\ref{fig:app-epochs}, for two fixed values of the noise level, how these metrics evolve through training.
\subsection{Perturbed shortest path}
\label{APP:short}
In this experiment, we have followed the setup of \citet{vlastelica2019differentiation}, to obtain comparable results. We have replicated the network that they use based on Resnet18, and followed their optimization procedure, using Adam with the same learning rate schedule, changing at epochs 30 and 40 out of 50. We also included the baseline that they used, based on training the same network without an optimizer layer. These results are obtained by using the implementation code that they provide. 

We minimize the weights of this model for our proposed Fenchel-Young loss (see Definition \ref{definition:FY_loss}).

The perfect accuracy metric measures the percentage of test instances for which an exactly optimal path is recovered, and the cost ratio to optimal metric measures the ratio between the total cost of the path proposed by taking the shortest path for proposed costs $\theta_i = g_w(x_i)$ (after training) to the total cost of the path with true costs (see Figure~\ref{FIG:shortestpathperf}).
\begin{table}[p]
\caption{Spearman correlation on $21$ datasets averaged by $10$-fold
    cross-validation. The learning rate is chosen from $(10^{-3}, 10^{-2},
    10^{-1})$ by grid search over $5$-fold cross-validation.}
\begin{small}
\begin{center}
\begin{tabular}{lcccc}
\toprule
Dataset & Perturbed FY & Perturbed $+$ Squared loss & Squared loss & Blackbox
loss \\
\midrule
authorship & 0.95 $\pm$ 0.01 & 0.28 $\pm$ 0.17 & {\bf 0.96} $\pm$ 0.01 & 0.76 $\pm$ 0.04 \\
bodyfat & 0.35 $\pm$ 0.07 & 0.23 $\pm$ 0.09 & {\bf 0.36} $\pm$ 0.08 & 0.32 $\pm$ 0.07 \\
calhousing & {\bf 0.26} $\pm$ 0.02 & 0.13 $\pm$ 0.07 & {\bf 0.26} $\pm$ 0.01 & 0.17 $\pm$ 0.06 \\
cold & 0.05 $\pm$ 0.04 & 0.00 $\pm$ 0.04 & {\bf 0.09} $\pm$ 0.04 & 0.04 $\pm$ 0.04 \\
cpu-small & {\bf 0.52} $\pm$ 0.01 & 0.44 $\pm$ 0.05 & 0.50 $\pm$ 0.01 & 0.50 $\pm$ 0.01 \\
diau & 0.22 $\pm$ 0.03 & 0.12 $\pm$ 0.05 & {\bf 0.26} $\pm$ 0.03 & 0.25 $\pm$ 0.03 \\
dtt & 0.11 $\pm$ 0.04 & 0.03 $\pm$ 0.06 & {\bf 0.15} $\pm$ 0.04 & 0.11 $\pm$ 0.04 \\
elevators & {\bf 0.79} $\pm$ 0.01 & 0.67 $\pm$ 0.05 & 0.77 $\pm$ 0.01 & 0.76 $\pm$ 0.02 \\
fried & {\bf 1.00} $\pm$ 0.00 & 0.82 $\pm$ 0.10 & 0.99 $\pm$ 0.00 & {\bf 1.00} $\pm$ 0.00 \\
glass & {\bf 0.88} $\pm$ 0.05 & 0.81 $\pm$ 0.09 & 0.86 $\pm$ 0.05 & 0.83 $\pm$ 0.05 \\
heat & 0.03 $\pm$ 0.03 & 0.01 $\pm$ 0.03 & {\bf 0.06} $\pm$ 0.02 & 0.03 $\pm$ 0.03 \\
housing & {\bf 0.75} $\pm$ 0.03 & 0.65 $\pm$ 0.07 & 0.70 $\pm$ 0.03 & 0.71 $\pm$ 0.03 \\
iris & {\bf 0.81} $\pm$ 0.09 & 0.78 $\pm$ 0.21 & {\bf 0.81} $\pm$ 0.08 & 0.70 $\pm$ 0.08 \\
pendigits & 0.95 $\pm$ 0.00 & 0.82 $\pm$ 0.06 & 0.94 $\pm$ 0.00 & {\bf 0.96} $\pm$ 0.00 \\
segment & {\bf 0.95} $\pm$ 0.01 & 0.78 $\pm$ 0.06 & 0.94 $\pm$ 0.00 & 0.93 $\pm$ 0.00 \\
spo & 0.16 $\pm$ 0.02 & 0.07 $\pm$ 0.03 & {\bf 0.18} $\pm$ 0.02 & {\bf 0.18} $\pm$ 0.02 \\
stock & {\bf 0.77} $\pm$ 0.05 & 0.56 $\pm$ 0.23 & 0.75 $\pm$ 0.03 & 0.63 $\pm$ 0.07 \\
vehicle & {\bf 0.87} $\pm$ 0.03 & 0.69 $\pm$ 0.09 & 0.84 $\pm$ 0.03 & 0.79 $\pm$ 0.03 \\
vowel & 0.73 $\pm$ 0.03 & 0.70 $\pm$ 0.03 & 0.73 $\pm$ 0.02 & {\bf 0.74} $\pm$ 0.02 \\
wine & 0.94 $\pm$ 0.03 & 0.85 $\pm$ 0.17 & {\bf 0.96} $\pm$ 0.03 & 0.86 $\pm$ 0.08 \\
wisconsin & 0.75 $\pm$ 0.03 & 0.56 $\pm$ 0.07 & {\bf 0.78} $\pm$ 0.03 & 0.75 $\pm$ 0.04 \\
\bottomrule
\end{tabular}
\end{center}
\end{small}
\label{tab:ranking_results}
\end{table}

\begin{table}[p]
\caption{Spearman correlation on $21$ datasets averaged by $10$-fold
    cross-validation. The learning rate \textbf{and} the temperature
    $\varepsilon$ are chosen from $(10^{-3}, 10^{-2},
    10^{-1})$ by grid search over $5$-fold cross-validation.}
\begin{small}
\begin{center}
\begin{tabular}{lcccc}
\toprule
Dataset & Perturbed FY & Perturbed $+$ Squared loss & Squared loss & Blackbox
loss \\
\midrule
authorship & 0.95 $\pm$ 0.01 & 0.93 $\pm$ 0.02 & {\bf 0.96} $\pm$ 0.01 & 0.76 $\pm$ 0.04 \\
bodyfat & 0.35 $\pm$ 0.07 & 0.34 $\pm$ 0.08 & {\bf 0.36} $\pm$ 0.08 & 0.32 $\pm$ 0.07 \\
calhousing & {\bf 0.26} $\pm$ 0.02 & 0.25 $\pm$ 0.04 & {\bf 0.26} $\pm$ 0.02 & 0.17 $\pm$ 0.06 \\
cold & 0.08 $\pm$ 0.04 & 0.08 $\pm$ 0.04 & {\bf 0.09} $\pm$ 0.04 & 0.04 $\pm$ 0.04 \\
cpu-small & 0.53 $\pm$ 0.01 & {\bf 0.54} $\pm$ 0.02 & 0.50 $\pm$ 0.02 & 0.50 $\pm$ 0.01 \\
diau & {\bf 0.26} $\pm$ 0.03 & {\bf 0.26} $\pm$ 0.02 & {\bf 0.26} $\pm$ 0.03 & 0.25 $\pm$ 0.03 \\
dtt & 0.14 $\pm$ 0.04 & 0.13 $\pm$ 0.04 & {\bf 0.15} $\pm$ 0.04 & 0.11 $\pm$ 0.04 \\
elevators & {\bf 0.80} $\pm$ 0.01 & 0.79 $\pm$ 0.01 & 0.77 $\pm$ 0.01 & 0.76 $\pm$ 0.02 \\
fried & {\bf 1.00} $\pm$ 0.00 & 0.99 $\pm$ 0.01 & 0.99 $\pm$ 0.00 & {\bf 1.00} $\pm$ 0.00 \\
glass & {\bf 0.88} $\pm$ 0.05 & 0.84 $\pm$ 0.06 & 0.86 $\pm$ 0.06 & 0.83 $\pm$ 0.05 \\
heat & {\bf 0.06} $\pm$ 0.03 & 0.05 $\pm$ 0.03 & {\bf 0.06} $\pm$ 0.02 & 0.03 $\pm$ 0.03 \\
housing & {\bf 0.76} $\pm$ 0.03 & 0.75 $\pm$ 0.03 & 0.70 $\pm$ 0.04 & 0.71 $\pm$ 0.03 \\
iris & 0.80 $\pm$ 0.12 & {\bf 0.86} $\pm$ 0.11 & 0.81 $\pm$ 0.08 & 0.70 $\pm$ 0.08 \\
pendigits & {\bf 0.96} $\pm$ 0.00 & 0.95 $\pm$ 0.00 & 0.94 $\pm$ 0.00 & {\bf 0.96} $\pm$ 0.00 \\
segment & {\bf 0.95} $\pm$ 0.00 & 0.94 $\pm$ 0.01 & 0.94 $\pm$ 0.01 & 0.93 $\pm$ 0.00 \\
spo & {\bf 0.18} $\pm$ 0.02 & {\bf 0.18} $\pm$ 0.02 & {\bf 0.18} $\pm$ 0.02 & {\bf 0.18} $\pm$ 0.02 \\
stock & {\bf 0.78} $\pm$ 0.07 & 0.70 $\pm$ 0.21 & 0.75 $\pm$ 0.03 & 0.63 $\pm$ 0.07 \\
vehicle & {\bf 0.89} $\pm$ 0.02 & 0.86 $\pm$ 0.03 & 0.84 $\pm$ 0.03 & 0.79 $\pm$ 0.03 \\
vowel & 0.74 $\pm$ 0.02 & {\bf 0.75} $\pm$ 0.03 & 0.73 $\pm$ 0.02 & 0.74 $\pm$ 0.02 \\
wine & 0.95 $\pm$ 0.04 & 0.91 $\pm$ 0.07 & {\bf 0.96} $\pm$ 0.04 & 0.86 $\pm$ 0.08 \\
wisconsin & {\bf 0.78} $\pm$ 0.03 & 0.77 $\pm$ 0.03 & 0.77 $\pm$ 0.03 & 0.75 $\pm$ 0.04 \\
\bottomrule
\end{tabular}
\end{center}
\end{small}
\label{tab:ranking_results_tuned}
\end{table}

\begin{figure*}[ht!]
\begin{center}
\includegraphics[width=\textwidth]{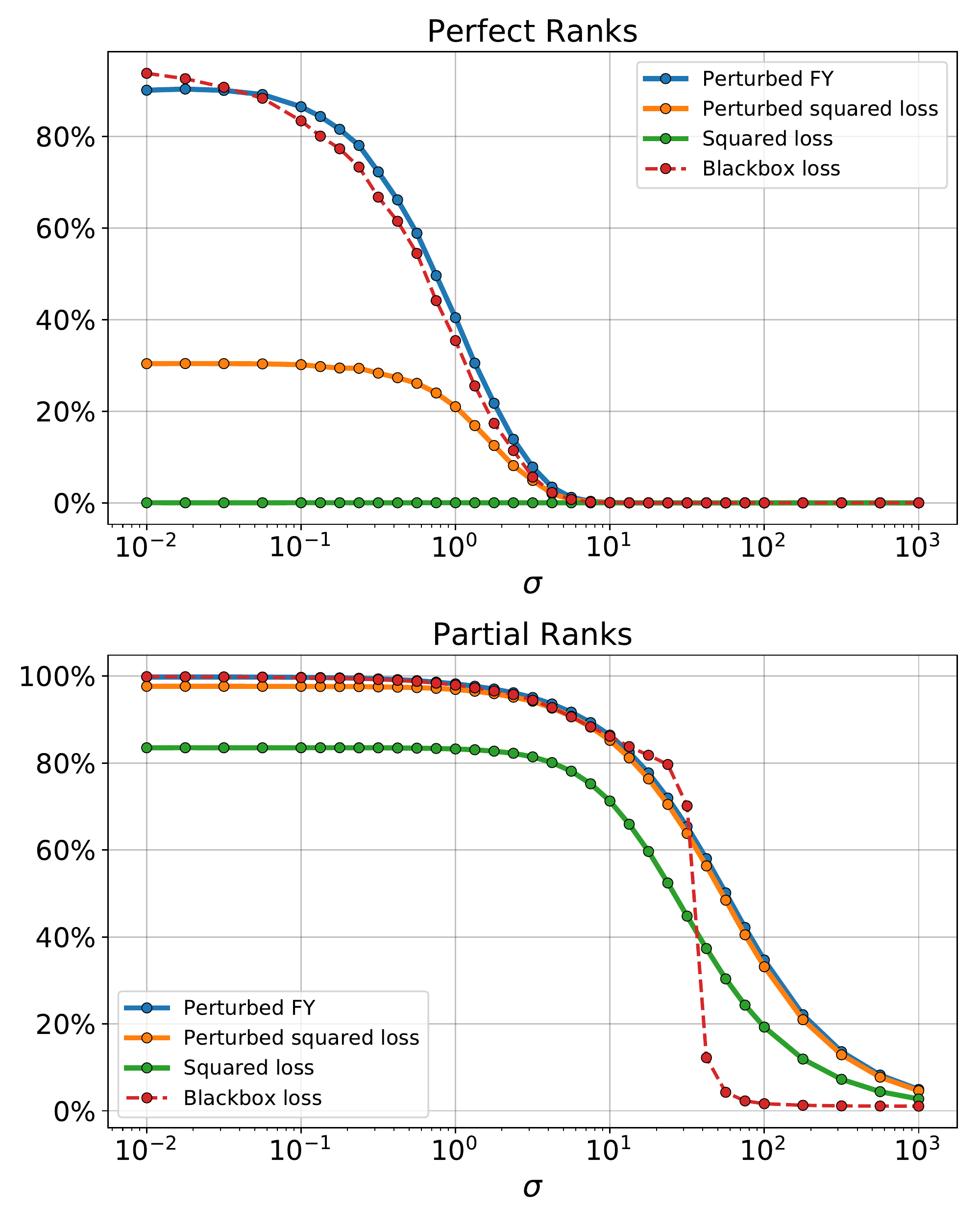}

\vspace{-0.1cm}
\caption{As in Figure~\ref{FIG:artificial}, we show here both ({\bf Top}) the average number of instances with exactly correct ranks (perfect ranks) for all 100 labels ({\bf Bottom}) the average number of correctly ranked labels (partial ranks). In both for different values of $\sigma \in [10^{-2}, 10^3]$, for four methods. \label{fig:partial-artificial}}
\end{center}
\end{figure*}

\begin{figure*}[ht!]
\begin{center}
\begin{tabular}{cc}
\includegraphics[width=0.5\textwidth]{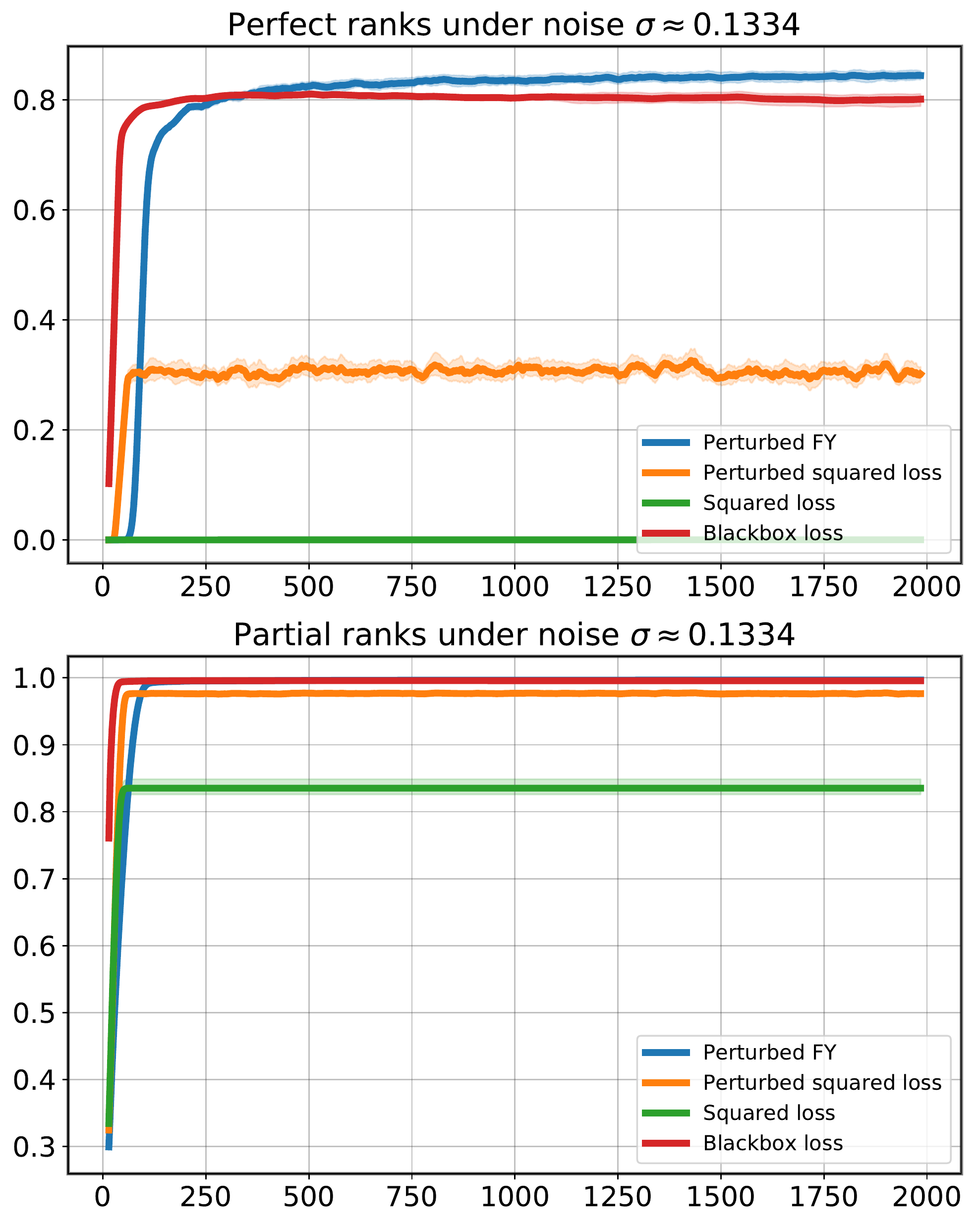}
&
\includegraphics[width=0.5\textwidth]{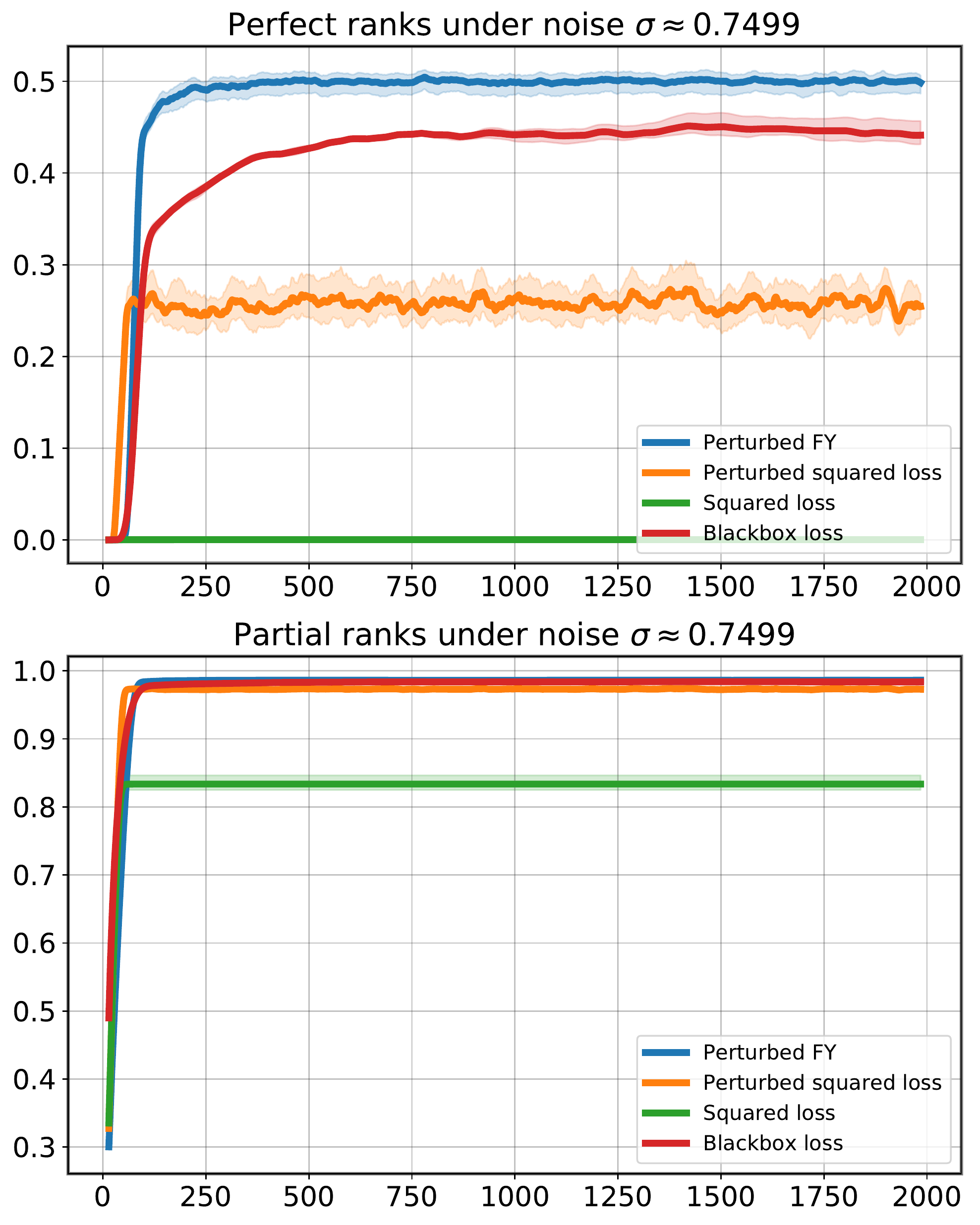}
\end{tabular}
\caption{We report the same metrics as in Figure~\ref{fig:partial-artificial} of perfect ranks and partial ranks at two fixed noise levels, as a function of the number of epochs ({\bf Left}) for $\sigma \approx 0.1334$ ({\bf Left}) for $\sigma \approx 0.7449$.\label{fig:app-epochs}}
\end{center}
\end{figure*}

\end{document}